%% file: kdd22.tex
\documentclass{article}

\usepackage[preprint,nonatbib]{neurips_2022}

\usepackage[bookmarks=false,colorlinks=true,urlcolor=black,citecolor={green!60!black}]{hyperref}       

\usepackage[utf8]{inputenc} 
\usepackage[T1]{fontenc}    
\usepackage{url}            
\usepackage{booktabs}       
\usepackage{amsfonts}       
\usepackage{nicefrac}       
\usepackage{microtype}      
\usepackage{xcolor}         
\usepackage{wrapfig}
\usepackage{caption}
\usepackage{microtype}
\usepackage{graphicx}
\usepackage{booktabs} 
\usepackage{algpseudocode}
\usepackage{algorithm}
\usepackage{algcompatible}
\usepackage{amsmath,bm}
\usepackage{multirow}
\usepackage{enumitem}
\usepackage{xspace}
\usepackage[most]{tcolorbox}
\usepackage{hyperref}
\usepackage{caption}
\usepackage{subcaption}

\usepackage{amsmath}
\usepackage{amssymb}
\usepackage{mathtools}
\usepackage{amsthm}
\usepackage{color, colortbl}

\usepackage[capitalize,noabbrev]{cleveref}

\theoremstyle{plain}
\newtheorem{theorem}{Theorem}[section]

\newtheorem{lemma}[theorem]{Lemma}

\theoremstyle{definition}
\newtheorem{definition}[theorem]{Definition}

\theoremstyle{remark}
\newtheorem{remark}[theorem]{Remark}

\usepackage[textsize=tiny]{todonotes}
\makeatletter
\DeclareRobustCommand\onedot{\futurelet\@let@token\bmv@onedotaux}
\def\bmv@onedotaux{\ifx\@let@token.\else.\null\fi\xspace}
\def\eg{\emph{e.g}\onedot} 
\def\ie{\emph{i.e}\onedot} 
 
\def\etc{\emph{etc}\onedot} \def\vs{\emph{vs}\onedot}
\def\wrt{w.r.t\onedot}

\def\etal{\emph{et al}\onedot}

\makeatletter
\def\ps@myheadings{%
    \let\@oddfoot\@empty\let\@evenfoot\@empty
    \def\@evenhead{\thepage\hfil\slshape\leftmark}%
    \def\@oddhead{{\slshape\rightmark}\hfil\thepage}%
    \let\@mkboth\@gobbletwo
    \let\sectionmark\@gobble
    \let\subsectionmark\@gobble
    }
  \if@titlepage
  \renewcommand\maketitle{\begin{titlepage}%
  \let\footnotesize\small
  \let\footnoterule\relax
  \let \footnote \thanks
  \null\vfil
  \vskip 60\p@
  \begin{center}%
    {\LARGE \@title \par}%
    \vskip 3em%
    {\large
     \lineskip .75em%
      \begin{tabular}[t]{c}%
        \@author
      \end{tabular}\par}%
      \vskip 1.5em%
    {\large \@date \par}
  \end{center}\par
  \@thanks\@notice
  \vfil\null
  \end{titlepage}%
  \setcounter{footnote}{0}%
}
\else
\renewcommand\maketitle{\par
  \begingroup
    \renewcommand\thefootnote{\@fnsymbol\c@footnote}%
    \def\@makefnmark{\rlap{\@textsuperscript{\normalfont\color{black}\@thefnmark}}}%
    \long\def\@makefntext##1{\parindent 1em\noindent
            \hb@xt@1.8em{%
                \hss\@textsuperscript{\normalfont\@thefnmark}}##1}%
    \if@twocolumn
      \ifnum \col@number=\@ne
        \@maketitle
      \else
        \twocolumn[\@maketitle]%
      \fi
    \else
      \newpage
      \global\@topnum\z@   
      \@maketitle
    \fi
    \thispagestyle{plain}\@thanks\@notice
  \endgroup
  \setcounter{footnote}{0}%
}
\makeatother

\usepackage{multibib}
\newcites{latex}{References}

\title{COSTA: Covariance-Preserving Feature Augmentation for Graph Contrastive Learning}

\author{%
  Yifei Zhang$^1$, Hao Zhu$^{2,3}$, Zixing Song$^1$, Piotr Koniusz$^{3,2}$, Irwin King$^1$\\
  The Chinese University of Hong Kong$^1$, Hong Kong SAR, China\\
Australian National University$^2$ and Data61/CSIRO$^3$, Canberra, Australia\\
 \texttt{\{yfzhang,zxsong,king\}@cse.cuhk.edu.hk}\\ \texttt{allenhaozhu@gmail.com; piotr.koniusz@data61.csiro.au}\\
}

\begin{document}

\maketitle

\begin{abstract}
Graph contrastive learning (GCL) improves graph representation learning, leading to SOTA on various downstream tasks. The graph augmentation step is a vital but scarcely studied step of GCL. In this paper, we show that the node embedding obtained via the graph augmentations is highly biased, somewhat limiting contrastive models from learning discriminative features for downstream tasks.Thus, instead of investigating graph augmentation in the input space, we alternatively propose to perform augmentations on the hidden features (feature augmentation). Inspired by  so-called matrix sketching, we propose \textbf{COSTA}, a novel \textbf{CO}variance-pre\textbf{S}erving fea\textbf{T}ure space \textbf{A}ugmentation framework for GCL, which generates augmented features by maintaining a ``good sketch'' of original features. To highlight the superiority of feature augmentation with {COSTA}, we investigate a single-view setting (in addition to multi-view one) which conserves memory and computations. We show that the feature augmentation with COSTA achieves comparable/better results than graph augmentation based models.
%
\end{abstract}

\input{Content/Introduction}
\input{Content/RelatedWork}
\input{Content/Method}

\input{Content/Experiment}
\input{Content/Conclusion}

\section{Acknowledgments}
This work was supported by the National Key Research and Development Program of China (No. 2018AAA0100204) and the Research Grants Council of the Hong Kong Special Administrative Region, China (CUHK 2410021, Research Impact Fund, No. R5034-18).

\bibliographystyle{plain}
\bibliography{kdd}
\appendix

\vspace{0.3cm}
\section*{------ Appendices ------}

\vspace{0.3cm}

\input{Content/Appendix}

\end{document}

%% file: Content/Introduction.tex
\section{Introduction}
\label{sec:intro}
Many Graph Neural Networks (GNNs)~\cite{kipf2016semi,deeper_look2,DBLP:conf/www/ZhangZMKK22,DBLP:conf/cikm/SongMZK21,9737635,yang2020featurenorm} focus on (semi-)supervised learning, which requires access to abundant labels. Recent trends in Self-Supervised Learning (SSL) have resulted in several methods that do not require labels~\cite{kipf2016variational,hamilton2017inductive}. Among SSL methods, Contrastive Learning (CL) already achieved comparable performance with its supervised counterparts on many tasks~\cite{chen2020simple,gao2021simcse}. Recently, CL has been applied to the graph domain. A typical Graph Contrastive Learning (GCL) method constructs multiple graph views by stochastic augmentation of the input to learn representations by contrasting positive samples with negative samples~\cite{zhu2020deep,peng2020graph,zhu2021graph}. However, the irregular structure of graphs complicates the adaptation of augmentation techniques used on images and prevents  extending of theoretical analysis for vision-based contrastive learning to graphs. Thus, many works focus on the empirical design of hand-crafted graph augmentations (GA) for graph contrastive learning (\ie, random edge/node/attribute dropping)~\cite{you2020graph,zhu2020deep,zhu2021graph}. Notably, some latest works point out that random data augmentations are problematic as their noise may not be relevant to downstream tasks~\cite{DBLP:journals/corr/AdversaGA,DBLP:conf/nips/whatGoodVeiw}. In certain scenarios (\ie, recommendation systems\cite{yang2022hrcf,chen2021attentive,chen2021modeling}), GCL  achieves the desired performance gain under extremely sparse GAs (with an edge dropout rate 0.9)~\cite{wu2021self} but   method~\cite{yu2022graph} achieves similar results without GAs. Such observations naturally raise the question: are there better augmentation strategies for GCL other than GA?

\begin{figure}[t]
    \centering
     \begin{subfigure}[t]{0.4\textwidth}
         \centering
         \includegraphics[width=\textwidth]{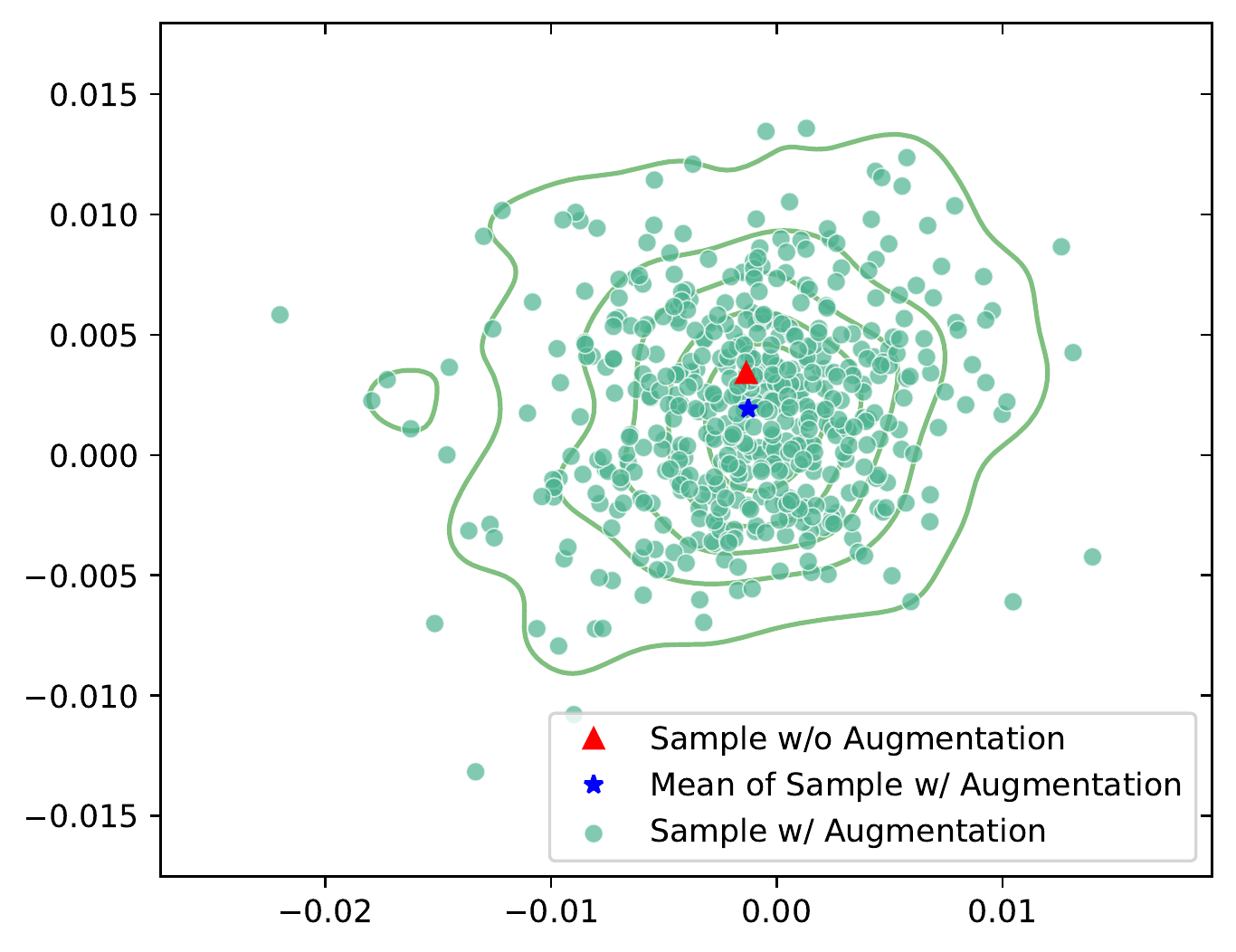}
         \caption{FA is unbiased.}
         \label{fig:issuseb}
     \end{subfigure}
     \begin{subfigure}[t]{0.4\textwidth}
         \centering
         \includegraphics[width=\textwidth]{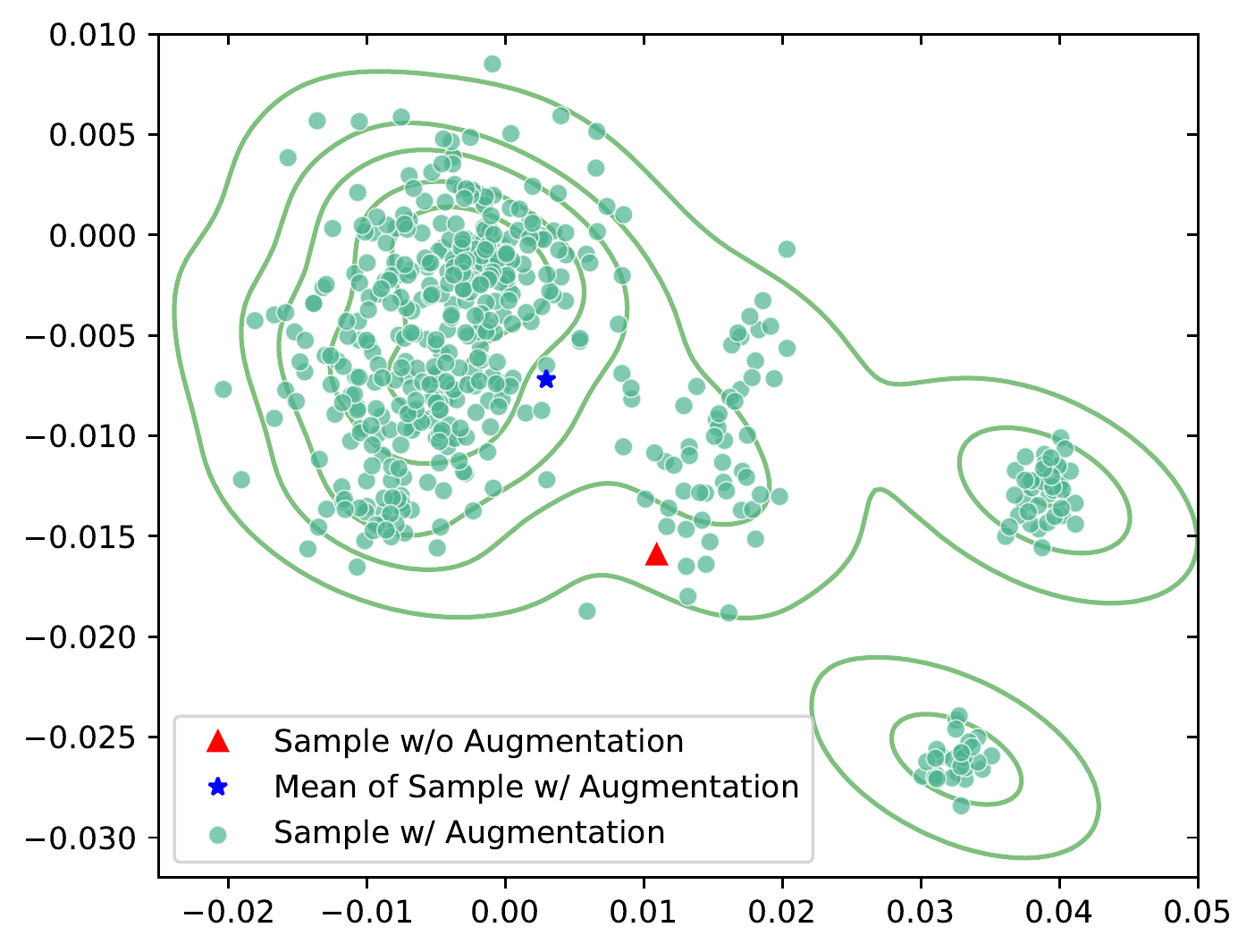}
         \caption{GA is biased.}
         \label{fig:issusea}
     \end{subfigure}
     \caption{The distribution of node embeddings on Cora is generated by 500$\times$ graph augmentations.~\ref{fig:issuseb} corresponds to the feature augmentation (Gaussian noise injection).~\ref{fig:issusea} corresponds to the graph augmentation (edge permutation \& attribute masking). We use 2D embeddings for visualization.
     }
     \label{fig:issuse}
\end{figure}

To this end, we show that the embeddings obtained with GA are highly biased compared to the embeddings obtained with feature augmentation (FA), that is, the embeddings obtained with FA (\eg, an injection of random noise into the embedding) exhibit the so-called weak law of large numbers (WLLN). Specifically, for any error $\varepsilon \geq 0$, $\lim_{k \rightarrow \infty} \mathcal{P}\left(\|\mathbb{E}(\tilde{\bm{x}}^{(k)}) - \bm{x}\|_1 > \varepsilon\right)=0$, where $\tilde{\bm{x}}^{(k)}$ denotes the embedding obtained by augmenting $\bm{x}$, and $\bm{x}$ is the original embedding without augmentation. For the i.i.d.  
random variables, as the sample size $k$ increases, the expectation of embedding after augmentation, $\mathbb{E}(\tilde{\bm{x}}^{(k)})$ , tends toward the real mean $\bm{x}$ (embedding without augmentation). 
In contrast to FA, GA violates the weak law of large numbers. 
%
%
%
%
As shown in Figure~\ref{fig:issuseb}, the population mean for the embedding obtained by the FA is in the densest region and in the proximity of the embedding of the original sample (without augmentation). In contrast, we cannot see such a trend in the case of the GA in Figure~\ref{fig:issusea}. In other words, GA introduces some bias, whereas FA produces unbiased embeddings. 

We assert that a successful contrastive objective should promote similarity/dissimilarity between features of encoded attributes by implicitly grouping/separating related/unrelated nodes according to their attribute space, respectively. However, as the GA strategy results in the bias (Figure~\ref{fig:issusea}), attraction/separation of embeddings in the feature space does not necessarily result in an optimal attraction/separation of desired nodes in the attribute space, which may result in suboptimal pre-training for downstream tasks. Figure~\ref{fig:conterexample} and Section~\ref{sec:motivation} further illustrate and motivate the above two scenarios.
Furthermore, the adoption of GA in GCL often increases the complexity as 
GCL compares the node features obtained from multiple views (\eg, multiple network streams) to obtain correlated views of the same graph. However, this strategy is prohibitive on large graphs as, in the worst-case scenario, multi-view GCL requires a time and space complexity quadratic \wrt the number of views and nodes. Thus, apart from the multi-view setting, we also investigate a single-view GCL setting.

\vspace{0.1cm}
\noindent\textbf{Our Contributions.} Instead of the GA, we propose to perform augmentation on the hidden feature vectors (feature augmentation). Inspired by matrix sketching, we propose \textbf{COSTA}, a novel \textbf{CO}variance-pre\textbf{S}erving fea\textbf{T}ure space \textbf{A}ugmentation framework for GCL, which produces augmented features by generating a ``good sketch'' of original features. To highlight the superiority of feature augmentation, apart from the multi-view setting, we show many results in the single-view setting, which conserves the memory usage and computations. We empirically show that COSTA (even the single-view variant, \ie, COSTA$_{SV}$) achieves comparable or better results than other GA strategies. 

\vspace{0.1cm}
Our contributions are threefold:
\renewcommand{\labelenumi}{\roman{enumi}.}
\hspace{-1.0cm}
\begin{enumerate}[leftmargin=0.6cm]
\item We point out the issue of bias introduced by the topology graph augmentation in the GCL framework, and we advocate feature augmentation strategies to prevent the aforementioned bias.

\item Inspired by matrix sketching, we propose COSTA, a simple and effective covariance-preserving feature augmentation framework for GCL, which generates augmented features by generating a ``good sketch'' (variance is bounded) of original features.

\item As an alternative to the multi-view GCL setting, we propose the single-view GCL setting, which produces equivalent or better results than the multi-view GCL while requiring less memory and incurring shorter computations.
\end{enumerate}

To our best knowledge, this is the first work which considers feature augmentation (in the single-view setting) in GCL with the matrix sketching step performing feature augmentation.

\begin{figure}[tp]
    \centering
    \begin{subfigure}[c]{0.6\textwidth}
         \centering
         \includegraphics[width=\textwidth]{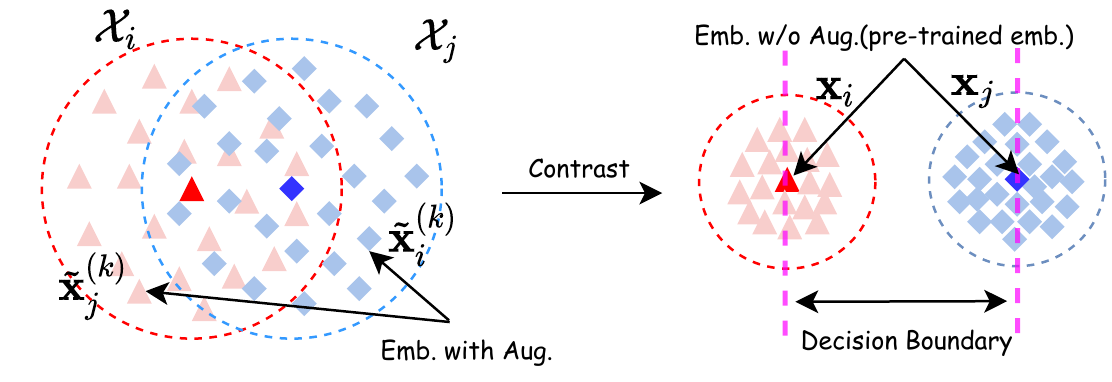}
         \caption{CL under unbiased augmentation.}
         \label{fig:auga}
     \end{subfigure}
     \begin{subfigure}[c]{0.6\textwidth}
         \centering
         \includegraphics[width=\textwidth]{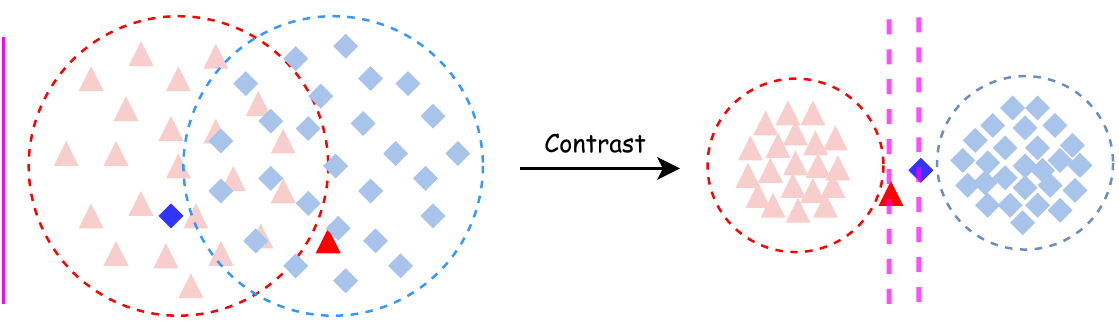}
         \caption{CL under biased augmentation.}
         \label{fig:augb}
     \end{subfigure}
    \caption{({Fig.~\ref{fig:auga}}) Our strategy results in an unbiased augmentation strategy. (Fig.~\ref{fig:augb}) A counter-example illustrates the problem of using a biased augmentation strategy in CL.}
    \label{fig:conterexample}
\end{figure}

%% file: Content/RelatedWork.tex
\section{Related Works}
\label{sec:relatedwork}
\subsection{Data Augmentation}
\vspace{0.1cm} 
\noindent\textbf{Input Space Augmentation},
 studied across many domains, usually refers to the augmentation performed in the input space. In computer vision, image transformations such as rotation, flipping, color jitters, translation, and noise injection~\cite{shorten2019survey} as well as more recent cut-off and random erasure~\cite{devries2017dataset} are very popular. In neural language processing, input space augmentations include token-level random augmentations such as synonym replacement, word swapping, word insertion, and deletion~\cite{wei2019eda}. In the graph domain, input space augmentation is referred to as graph augmentation. Attribute masking, edge permutation, and node dropout are common graph augmentation strategies~\cite{you2020graph}. Adaptive graph augmentations based on 
 node centrality and  PageRank centrality were studies by Zhu \etal~\cite{zhu2021graph} and Page \etal~\cite{page1999pagerank} respectively with the goal of masking different edges with varying probability. We discuss the negative effect of graph augmentation (\eg, edge and node removal) later in the text. 

\vspace{0.1cm} \noindent\textbf{Feature Augmentation }
strategies generate augmented samples in the feature space instead of the input space~\cite{feng2021survey}. Wang \etal~\cite{wang2019implicit} augment features in the hidden feature space, resulting in a feature representation that corresponds to another sample with the same class label but different semantics. So-called instance augmentations add perturbations to original instances~\cite{wang2019implicit}. Many few-shot learning approaches \cite{hariharan2017low} estimate the ``analogy transformations'' between examples of known classes to apply them to examples of novel classes.  
Finally, feature augmentations are popular in many research domains, \ie, semi-supervised learning, one-shot learning, and few-shot learning. However, no prior work has combined FA with contrastive learning in the graph domain the way COSTA performs FA.

\subsection{Graph Contrastive  Learning}
Inspired by contrastive methods in vision and NLP~\cite{he2020momentum, chen2020simple, gao2021simcse}, CL has also been adapted to the graph domain. By adapting DeepInfoMax~\cite{bachman2019learning} to graph representation learning, DGI~\cite{velickovic2019deep} learns  embedding 
by maximizing the mutual information to discriminate between nodes of original and corrupted graphs. REFINE \cite{zhu2021refine} uses a simple negative sampling term inspired by skip-gram models. Fisher-Bures Adversarial GCN \cite{uai_ke} uses adversarial perturbations of graph Laplacian. Inspired by SimCLR~\cite{chen2020simple}, GRACE~\cite{zhu2020deep}  correlates graph views by pushing closer representations of the same node in different views and pushing apart representations of different nodes. Another example of a SimCLR strategy is the recent GraphCL method~\cite{hafidi2020graphcl}. In contrast to GRACE, which learns node embedding, GraphCL learns embeddings for graph-level tasks. The above multi-view  methods suffer from the large memory and computational footprint, respectively. Although COLES~\cite{zhu2021contrastive} proposes a robust single-view GCL approach, it works the best with linear GNNs such as S\textsuperscript{2}GC~\cite{zhu2021simple}.
Thus, apart from multi-view COSTA, we also study the single-view GCL setting with FA. 

\section{Preliminaries}
\label{sec:SVGCL}
\subsection{Notations}
In this paper, a graph with node features is denoted as $G=(\mathcal{V}, \mathcal{E}, \bm{X})$, where $\mathcal{V}$ is the vertex set, $\mathcal{E}$ is the edge set, and $\bm{X} \in \mathbb{R}^{n \times d}$ is the feature matrix (\ie, the $i$-th row of $\bm{X}$ is the feature vector $\bm{x}_i$ of node $v_{i}$). 
Let $n=|\mathcal{V}|$ and $m=|\mathcal{E}|$ be the numbers of vertices and edges respectively. We use $\bm{A} \in\{0,1\}^{n \times n}$ to denote the adjacency matrix of $G$, \ie, the $(i, j)$-th entry in $\bm{A}$ is 1 if and only if there is an edge between $v_{i}$ and $v_{j}$.
The degree of a node $v_{i}$, denoted as $d_{i}$, is the number of edges incident with $v_{i}$.
The degree matrix $\bm{D}$ is a diagonal matrix, and its $i$-th diagonal entry is $d_{i}$. 
For a $d$-dimensional vector, $\bm{x} \in \mathbb{R}^d, \|\bm{x}\|_{2}$ is the Euclidean norm of $\bm{x}$.
We use ${x}_{i}$ to denote the $i$-th entry of $\bm{x}$, and $\operatorname{diag}(\bm{x}) \in \mathbb{R}^{d \times d}$ is a diagonal matrix such that the $i$-th diagonal entry is $x_{i}$. 
We use $\bm{A}_{i:}$ and $\bm{A}_{:i}$ to denote the $i$-th row and column of $\bm{A}$ respectively, and ${A}_{ij}$ for the $(i, j)$-th entry of $\bm{A}$. 
The trace of a square matrix $\bm{A}$ is denoted by $\operatorname{Tr}(\bm{A})$, which is the sum along the diagonal of $\bm{A}$. 
The singular value decomposition of $\bm{A}$ is denoted as $\bm{A} = \bm{U}\bm{\Sigma}\bm{V}^\top$ where $\bm{U}=\left[\bm{u}_{1}, \ldots, \bm{u}_{n}\right], \bm{\Sigma}=\operatorname{diag}\left (\sigma_{1}, \ldots, \sigma_{d}\right)$, and $\bm{V}=\left[\bm{v}_{1}, \ldots, \bm{v}_{d}\right]$. 
We use $\| \bm{A} \|_2$ to denote the spectral norm of $\bm{A}$, which is the largest singular value $\sigma_{\max}$. 
We use $\|\bm{A}\|_F$ for the Frobenius norm, which is $\|\bm{A}\|_{\mathrm{F}}=\sqrt{\sum_{i,j}\left|a_{i j}\right|^{2}}=\sqrt{\operatorname{
Tr}\left(\bm{A^\top A}\right)}=\sqrt{\sum_{i=1}^{d} \sigma_{i}^{2}(\bm{A})}$.

\subsection{Multi-view Graph Contrastive Learning (MV-GCL)}
Following the conventions presented in \cite{zhu2020deep, zhu2021graph},
MV-GCL learns node representations by maximizing the mutual information (MI) between views of the same graph.
Below, we introduce components of MV-GCL: (i) graph augmentation, (ii) GNN-based encoders, (iii) projection head, and (iv) a contrastive loss. 

\vspace{0.1cm}
\noindent\textbf{Graph Augmentation.} 
$\mathcal{T}_{GA}$ generates augmented  $(\tilde{\bm{A}}, \tilde{\bm{X}})$ by directly adding random perturbations to the original graph $({\bm{A}}, {\bm{X}})$. 
Different augmented graphs are constructed given one input $(\bm {A}, \bm {X})$, yielding correlated views $(\tilde{\bm{A}}_i, \tilde{\bm{X}}_i)$ that represent the augmented adjacent matrix and node features in the $i$-th view.
In the common GCL setting~\cite{zhu2020deep, zhu2021graph}, the graph structure is augmented via edge permutation. Node features are augmented via attribute masking.

\vspace{0.1cm}
\noindent\textbf{Graph Neural Network Encoders.} The GNN encoder $f_i:
\mathbb{R}^{n\times n} \times \mathbb{R}^{n \times d} \rightarrow \mathbb{R}^{n\times d}$ 
extracts hidden node features $\bm{H}_i\in \mathbb{R}^{n\times d}$ from the $i$-th augmented graph $(\tilde{\bm{A}}_i, \tilde{\bm{X}}_i)$. 
Usually, multiple encoders are applied to obtain the hidden node features $\bm{H}_i$ of different views as:
\begin{equation}
    \bm{H}_1 = f_1 (\tilde{\bm{A}}_1, \tilde{\bm{X}}_1), \cdots, \bm{H}_k = f_k (\tilde{\bm{A}}_k, \tilde{\bm{X}}_k).
\end{equation}
The GNN encoders are implemented as a two-layer Graph Convolution Network (GCN):
    \begin{equation}
    \begin{aligned}
    \mathrm{GCN}_{l} (\bm{X}, \bm{A}) &=\sigma\left (\hat{\bm{D}}^{-\frac{1}{2}} \hat{\bm{A}} \hat{\bm{D}}^{-\frac{1}{2}} \bm{X} \bm{W}_l\right), \\
    f (\bm{X}, \bm{A}) &=\mathrm{GCN}_{2}\left (\mathrm{GCN}_{1} (\bm{X}, \bm{A}), \bm{A}\right),
    \end{aligned}
    \end{equation}
where $\hat{\bm{A}}=\bm{A}+\bm{I}$ is the adjacency matrix with self-loops, $\bm{D}$ is the degree matrix, $\sigma (\cdot)$ is an activation function, \eg, $\operatorname{ReLU} (\cdot)=\max (0, \cdot)$, and $\bm{W}_l$ is a trainable weight matrix for the $l$-th layer.

\vspace{0.1cm}
\noindent\textbf{Feature Augmentation $\mathcal{T}_{FA}$.}
We apply feature augmentation on $\bm{H}$. We elaborate on the proposed FA and detail its properties in Section~\ref{sec:cpfa}. FA results in the augmented feature maps  fed into the projection head describe below.  

\vspace{0.1cm}
\noindent\textbf{Projection Head.} The projection head $\theta(\cdot)$ is a small network that maps representations to the space where contrastive loss is applied. It is implemented as a multi-layer perceptron (MLP) with one hidden layer to obtain $\bm{Z}_{i}=\theta\left (\bm{H}_{i}\right)=\bm{W}^{ (2)} \sigma\left (\bm{W}^{ (1)} \bm{H}_{i}\right)$, where $\sigma$ is the ReLU non-linearity. As described in~\cite{chen2020simple}, it is beneficial to define the contrastive loss on $\bm{Z}_{i}$ rather than $\bm{H}_{i}$.

\vspace{0.1cm}
\noindent\textbf{Contrastive Loss}. Let two feature matrices $\bm{U} \in \mathbb{R}^{n \times d}$ and $\bm{V} \in \mathbb{R}^{n\times d}$, where $\bm{U}=\bm{Z}_1$ and $\bm{V} = \bm{Z}_2$ are node features obtained from two different views. Then for any node $i$, its embedding generated in one view, $\bm{U}_{i:}$, is treated as the anchor, its embedding generated in another view, $\bm{V}_{i:}$, forms the positive sample. Remaining node embeddings $\bm{U}_{j:}$ and $\bm{V}_{j:}$ such that $j\neq i$ (from two views) are naturally regarded as negative samples. The contrastive loss function $\mathcal{L}$ for all positive pairs is defined as:
\begin{equation}
\label{eq:mvcontrastiveloss}
\begin{aligned}
&\mathcal{L} = \frac{1}{n}\sum_{i=1}^{n} \Big[{(\bm{U}^\top_{i:}, \bm{V}_{i:})}/\tau -\log \Big(
\sum_{j=1}^{n} e^{\left (\bm{U}_{i:}^\top, \bm{V}_{j:}\right) / \tau}+\sum_{j=1}^{n} e^{\left (\bm{U}^\top_{i:}, \bm{U}_{j:}\right) / \tau}\Big) \Big].
\end{aligned}
\end{equation}

Note that computing Eq.~(\ref{eq:mvcontrastiveloss}) is both memory costly and time consuming as it requires the computation of three large similarity matrices, $\bm{UV}^\top$, $\bm{UU}^\top$, $\bm{VU}^\top \in \mathbb{R}^{n\times n}$ for two views. Therefore, the memory consumption and runtime  depend on the number of views multiplied by the number of nodes, making such a multi-view setting challenging to run on large-scale graphs.

\vspace{0.1cm}
\noindent\textbf{Single-view Graph Contrastive Learning (SV-GCL).} To validate the effectiveness of graph augmentation and feature augmentation, apart from MV-GCL, we use a special case of multi-view GCL that shares the same augmented graph for two views and is thus equivalent to single-view Graph Contrastive Learning (SV-GCL). We note that SV-GCL has a computational advantage, \ie, only the features of a single view are calculated, and distances between features within the view. Distances between views are not needed. SV-GCL also provides a fairer way to compare the effectiveness of graph augmentations and feature augmentations as otherwise the multi-view setting would be the reason for the performance gain rather than the graph augmentation strategy. 

%% file: Content/Method.tex
\definecolor{beaublue}{rgb}{0.88,15,1}
\definecolor{blackish}{rgb}{0.2, 0.2, 0.2}
\definecolor{Gray}{gray}{0.9}
\definecolor{LightCyan}{rgb}{0.88,1,1}
\section{Methodology}
Section~\ref{sec:motivation}  presents our motivation. 
Section \ref{sec:cpfa} presents \textbf{COSTA}, \textbf{CO}variance pre\textbf{S}erving fea\textbf{T}ure space \textbf{A}ugmentation framework. Section~\ref{sec:matrixskeching} relates COSTA to the problem of matrix sketching, which generates desired augmented samples with theoretical guarantee. 
\begin{figure}[t]
    \centering
    \begin{subfigure}[c]{0.7\textwidth}
    \includegraphics[width=\textwidth]{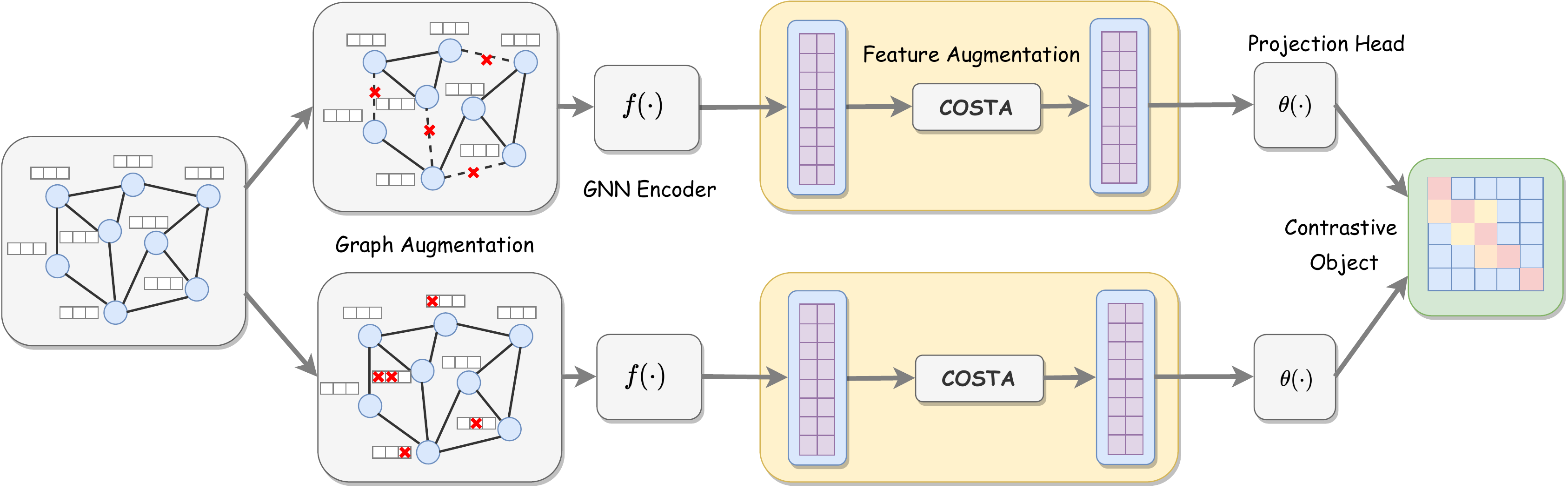}
    \caption{\label{fig:mvgcl} Multi-View GCL.}
    \end{subfigure}
    \begin{subfigure}[c]{0.7\textwidth}
    \includegraphics[width=\textwidth]{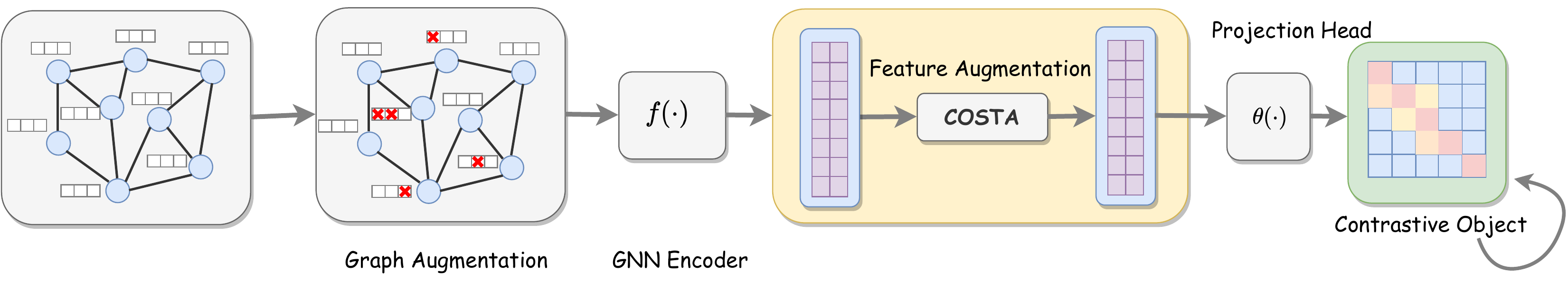}
    \caption{\label{fig:svgcl} Single-View GCL.}
    \end{subfigure}
    \label{fig:model}
    \caption{The illustrations of MV-GCL (standard) in Fig.~\ref{fig:mvgcl} and SV-GCL (simplified) in Fig~\ref{fig:svgcl}. For simplicity, the architecture of MV-GCL is shown using two views only and SV-GCL is the trivial case where two views are the same. MV-GCL contrast two views while SV-GCL perform self-contrast.}
\end{figure}
\subsection{Motivation}
\label{sec:motivation} 
We motivate COSTA with a simple experimental example, which shows that the node embedding obtained under graph augmentation is highly biased compared to feature augmentation. Inspired by WLLN (the weak law of large numbers explained in the introduction), below we quantify the bias introduced by the data augmentation $\mathcal{T}(\cdot)$ as follows. Let $\bm{x}_i$ be the original embedding of the $i$-th node and  $\tilde{\mathcal{X}}_i$ be its augmentation set where each embedding $\tilde{\bm{x}}^{(k)}_i\in\tilde{\mathcal{X}}_i$ is obtained by stochastic transformation, \ie, $\mathcal{T}(f(\cdot))$ or $f(\mathcal{T}(\cdot))$. Let the transformation distribution of $\bm{x}_i$ be  $\tilde{\mathcal{T}}(\bm{x}_i)$, then: 
\begin{equation}
      \text{Bias} (\mathcal{T} (\bm{x}_i))=\Big\| \mathbb{E}_{\tilde{\bm{x}}_i \sim \tilde{\mathcal{T}}(\bm{x}_i)} (\tilde{\bm{x}}_i) - \bm{x}_i \Big\|_2  \approx \Big\|  \frac{1}{|\tilde{\mathcal{X}}_i|}\sum_{k=1}^{|\tilde{\mathcal{X}}_i|} \tilde{x}^{(k)}_i\!  - \bm{x}_i \Big\|_2.
\label{eq:bias}
\end{equation}
Following Eq.~\eqref{eq:bias}, given a randomly chosen node (Cora dataset), we randomly generate $|\tilde{\mathcal{X}}| = 500$ augmented samples by the graph augmentation (GA) by edge permutation, and the feature augmentation (FA) by directly adding random noise to $\bm{x}$. In both cases, we use the same encoder to obtain the node embeddings. The output layer of the encoder has two dimensions to facilitate visualization. Figure~\ref{fig:issuse} depicts the distribution of node embeddings for both types of augmentations. Clearly, the expectation of node embeddings $\mathbb{E} (\mathcal{T}_{\text{FA}}({\bm{x}}))$ (the blue star in Figure~\ref{fig:issuseb}) obtained by the FA converges to its original embedding $\bm{x}$ (the red triangle in Figure~\ref{fig:issuseb}), whereas expectation of node embeddings $\mathbb{E} (\mathcal{T}_{\text{GA}}({\bm{x}}))$ obtained via the GA deviates far from $\bm{x}$, indicating the following:

\definecolor{beaublue}{rgb}{0.9, 0.95, 0.9}
\definecolor{blackish}{rgb}{0.2, 0.2, 0.2}
\begin{tcolorbox}[width=1.0\linewidth, colframe=blackish, colback=LightCyan, boxsep=0mm, arc=2mm, left=2mm, right=2mm, top=1mm, bottom=2mm]
\begin{equation}
    \operatorname{Bias}(\mathcal{T}_\text{GA}(\bm{x}))\gg\operatorname{Bias} (\mathcal{T}_\text{FA}(\bm{x})).
\end{equation}
\end{tcolorbox}

As shown, GA-based contrastive learning suffers from optimizing the biased setting. 
A standard contrastive loss aims to 
maximize the similarity of embeddings within the same augmentation set (\eg, $\tilde{\bm{x}}^{(1)}_i$, $\tilde{\bm{x}}^{(2)}_i \!\in \tilde{\mathcal{X}}_i$ ) and minimize the similarity of embeddings between different sets (\eg, $\tilde{\bm{x}}^{(1)}_i\!\in \tilde{\mathcal{X}}_i$ and $\tilde{\bm{x}}^{(1)}_j\!\in \tilde{\mathcal{X}}_j$). To perform well, such an objective requires the augmented embeddings to adhere to the unbiased case described above because as the bias tends to zero, the expectation of augmented embeddings converges to ${\bm{x}}_i$, \ie, $\mathbb{E}({\tilde{\bm{x}}}_i)\rightarrow{\bm{x}}_i$. Thus, pushing away $\tilde{\bm{x}}^{(k)}_i$ from $\tilde{\bm{x}}^{(k)}_j$ if $i\neq j$ separates embeddings of different instances, whereas the augmented embeddings $\tilde{\bm{x}}^{(k)}_i$  concentrate around  $\bm{x}_i$. 
In contrast, when the bias is large, \ie, $\|\mathbb{E}(\tilde{\bm{x}}_i) - \bm{x}_i\|_2 \gg 0 $, separating augmented embeddings of different instance (\ie, $\tilde{\bm{x}}_i^{(k)}$, $\tilde{\bm{x}}_j^{(k)}$, $i\neq j$) may not increase the discrimination of learned embeddings (\ie, $\bm{x}_i, \bm{x}_j)$ for downstream tasks. Figure~\ref{fig:conterexample} (bottom) illustrates such a case.

This motivates us to explore new ways of performing augmentations for GCL. Instead of explicitly eliminating the bias in the current GCL, we apply  feature augmentations as an alternative.

\subsection{Covariance-Preserving Feature Augmentation}
\label{sec:cpfa} 
Previous section indicates the graph augmentation produces the bias. We adopt the feature augmentation  for GCL loss because FA lets us  control the variance and reduce the bias. 
%
%

\begin{tcolorbox}[width=1.0\linewidth, colframe=blackish, colback=LightCyan, boxsep=0mm, arc=2mm, left=2mm, right=2mm, top=2mm, bottom=2mm]
Thus, we propose \textbf{the feature augmentation framework} in which the augmented feature matrix $\tilde{\bm{X}} \in \mathbb{R}^{k \times d}$, given the original feature matrix $\bm{X}\in \mathbb{R}^{n\times d}$, is obtained via:
\begin{equation}
\begin{aligned}
    \label{eq:coverr}
    &\tilde{\bm{X}} = \bm{PX} + \bm{E},\\
     \text{such that }\;&\|\bm{X}^{\top}\bm{X}-\tilde{\bm{X}}^{\top}\tilde{\bm{X}}\|_2 \leq \varepsilon \text{Tr} (\bm{X}^\top\bm{X}).
\end{aligned}
\end{equation}
\end{tcolorbox}
\noindent $\bm{P}\in \mathbb{R}^{k \times n}$ in Eq. \eqref{eq:coverr} denotes an affine transformation, $\bm{E}$ is the random noise matrix and $\varepsilon$ is the error which controls the quality of approximation  $\|\bm{X}^{\top}\bm{X}-\tilde{\bm{X}}^{\top}\tilde{\bm{X}}\|_2$. Note that the affine transformation can be either deterministic or stochastic.

\vspace{0.1cm}
\noindent\textbf{Connection to the Gaussian Noise Injection.} One special case of COSTA is the Gaussian noise injection \cite{wang2019implicit,yu2022graph} which produces the augmented feature matrix $\tilde{\bm{X}}$ by adding random noise sampled 
as $\tilde{\bm{X}}_{i:} \sim \!~\! \mathcal{N} (\bm{X}_{i:}, \varepsilon\bm{I})$, where $\varepsilon\geq 0$ controls the strength of noise.
This is equivalent to setting $\bm{P} = \bm{I}$ and $\bm{E}_{i:} \sim \mathcal{N} (0, \varepsilon\bm{I})$ in Eq.~(\ref{eq:coverr}). 
\subsection{Feature Augmentation via Matrix Sketching}
\label{sec:matrixskeching}
\begin{definition}[Matrix Sketching~\cite{liberty2013simple}]
\label{def:matrixsketch}
Let $\bm{X} \in \mathbb{R}^{n \times d}$ be the given feature matrix, $\bm{P} \in \mathbb{R}^{k \times n}$ be a sketching matrix, \eg, random projection or row selection matrix. The sketch of $\bm{X}$ is defined as $\tilde{\bm{X}}=\bm{P} \bm{X} \in \mathbb{R}^{k \times d}$. Usually, $\tilde{\bm{X}}$ contains fewer rows than $\bm{P}$, where $k\ll n$ but $\tilde{\bm{X}}$ still preserves many properties of $\bm{P}$.
\end{definition}
Eq.~(\ref{eq:coverr}) 
performs matrix sketching.
Obtaining the augmented feature matrix $\tilde{\bm{X}}$ requires  a good  sketch of $\bm{X}$ such that second-order statistics of the original and sketched matrices are similar. 
In what follows, we use SVD, random row selection, or random projection to form a sketch of $\bm{X}$. We 
prove that $\tilde{\bm{X}}$ obtained by sketching satisfies $\| \bm{X}^\top\bm{X}-\tilde{\bm{X}}^\top\tilde{\bm{X}} \|_2 \leq \varepsilon \operatorname{Tr} (\bm{X}^\top\bm{X})$ for small $\varepsilon\geq 0$.

\vspace{0.1cm} 
\noindent\textbf{Matrix Sketching via SVD.} One solution for Eq.~(\ref{eq:coverr}) can be obtained through the singular value decomposition ($SVD$) where:
\begin{equation}
    \bm{P} = \bm{U}^\top, \bm{X} = \bm{U} \bm{\Sigma} \bm{V}^\top.
\end{equation}
\begin{lemma}
Let $\tilde{\bm{X}} = \bm{P}\bm{X}$ and  $\bm{X} = \bm{U} \bm{\Sigma} \bm{V}^\top\!$ where $\,\bm{U}=\left [\bm{u}_{1}, \ldots, \bm{u}_{n}\right]$, $\bm{\Sigma }=\operatorname{diag}\left (\sigma_{1}, \ldots, \sigma_{d}\right)$, $\bm{V}=\left[\bm{v}_{1}, \ldots, \bm{v}_{d}\right]$. Then $\|\bm{X}^\top\bm{X}-\tilde{\bm{X}}^\top\tilde{\bm{X}}\|_2$ is bounded as:
\begin{equation}
    \|\bm{X}^\top\bm{X}-\tilde{\bm{X}}^\top\tilde{\bm{X}}\|_2 \leq \frac{\sigma_{k+1}}{\sigma_{\max}} \operatorname{Tr} (\bm{X}^\top\bm{X}).
\end{equation}
\label{lemma:svd}
\end{lemma}
\begin{proof}
See Appendix~\ref{proof:svd}.
\end{proof}
\begin{remark}
The upper bound of  $\|\bm{X}^\top \bm{X} - \tilde{\bm{X}}^\top\tilde{\bm{X}})\|_2$ is controlled by the $(k+1)$-th largest eigenvalue $\sigma_{k+1}$. Usually, $\frac{\sigma_{k+1}}{\sigma_{\max}}$ is small as $\sigma_{k+1} \ll \sigma_{\max}$ even when $k$ is small. However, SVD is computationally intensive and unsuitable for decomposition of large feature matrices.
\end{remark}

\vspace{0.1cm} \noindent\textbf{Random Row Selection (RS).}
Randomized algorithms  trade accuracy for efficiency and strive for high accuracy and low runtime. 
A sketch of a matrix $\tilde{\bm{X}}$ can be constructed via randomly stacking the rows of the original matrix $\bm{X}$.
Random row selection employs a small subset of rows based on a pre-defined probability distribution $\mathcal{P}(i)$ to form a sketch. The random assignment matrix $\bm{P}\in \mathbb{R}^{k \times n}$  stacks one-hot vectors, \ie, $\bm{P}$ = $\{\bm{e}_{i} \in \mathbb{R}^n| \mathcal{P} (i)=\frac{\|\bm{X}_{i:}\|_2}{\|\bm{X}\|_F}\} \in \mathbb{R}^{k \times n}$, where $\bm{e}_i$ indicates that the $i$-th row is selected.
\begin{lemma}
 Let $\bm{X} \in \mathbb{R}^{n \times d}$. Let $\tilde{\bm{X}} \in \mathbb{R}^{m \times d}$ be a matrix whose rows are randomly selected from rows of  $\bm{X} \in \mathbb{R}^{n \times d}$. It holds that:
\begin{equation}
 \mathcal{P}\left(\|\bm{X}^\top\bm{X} - \tilde{\bm{X}}^\top\tilde{\bm{X}}\|_2 \leq \varepsilon \operatorname{Tr} (\bm{X}^\top\bm{X})\right) \geq 1 - e^ {\left(-\frac{(\varepsilon \sqrt{k}-1)^2}{8}\right)}.
\end{equation}
\label{lemma:randomselect}
\end{lemma}
\begin{proof}
See Appendix~\ref{proof:randomselect}.
\end{proof}
\begin{remark}
The failure probability $\delta_{RS} = e ^{(-\frac{(\varepsilon \sqrt{k}-1)^2}{8})}$ is exponentially decreasing with the error $\varepsilon$ meaning that we can bound $\|\bm{X}^\top\bm{X} - \tilde{\bm{X}}^\top\tilde{\bm{X}}\|_2$ given small $\varepsilon\geq 0$ with a high probability $1-\delta_{RS}$.
\end{remark}

\vspace{0.1cm} \noindent\textbf{Random Projection (RP).}
A sketch of matrix can be RP. The projection matrix is defined as $\bm{P} \in \mathbb{R}^{n\times k}$ whose entry $p_{ij}$ is sampled from $\mathcal{N} (0, 1)$. %
Ideally, we expect $\bm{P}$ to provide a stable sketch that approximately preserves the distance between all pairs of columns in the original matrix. As the computation of dense matrix $\bm{P}$ is time-consuming, a sparse version from Appendix~\ref{sec:vsrp}) can be used.  

\begin{lemma}
\label{lemma:rp}
 let $\tilde{\bm{X}} = \frac{1}{\sqrt{k}}\bm{P}\bm{X}$ where $p_{ij}$ 
 is the $(i,j)$-th element of $\bm{P}$, and 
 $p_{ij}\sim \mathcal{N}(0,1)$. For $\varepsilon \in (0,1)$, it holds that:
\begin{equation}
\label{eq:rpbound}
 \mathcal{P}\left(\|\bm{X}^\top\bm{X} - \tilde{\bm{X}}^\top\tilde{\bm{X}}\|_2 \leq \varepsilon \operatorname{Tr} (\bm{X}^\top\bm{X})\right) \geq 1 - e^{-\frac{\varepsilon^2k}{8}}.
\end{equation}
\end{lemma}
\begin{proof}
See Appendix~\ref{proof:rp}.
\end{proof}
\begin{remark}
Note that the failure probability of the random projection $\delta_{PR}$ is less than $\delta_{RS}$ when $k > \varepsilon^{-2}$.
\end{remark}

%% file: Content/Experiment.tex
\section{Experiments}
Below, we  provide 
details of experimental settings, and we discuss our results. We  answer the following research questions (\textbf{RQ}):
\begin{itemize}
    \item \textbf{RQ1:} What is the bias problem in graph augmentation? 
    \item \textbf{RQ2:} Does the proposed 
    feature augmentation work for problems of practical interest? How does its accuracy/speed compare with   MV-GCL models that adopt the graph augmentation strategy, and with other models? Does SV-GCL perform well in comparison to MV-GCL models? 
    \item \textbf{RQ3:} What is the performance of the feature augmentation given different matrix sketching schemes? What are the major factors that contribute to the success of the proposed feature augmentation method?
    \item \textbf{RQ4:} How is the effectiveness affected by the number of augmented samples?
\end{itemize}
\begin{figure}[t]
    \centering
    \begin{minipage}[t]{0.35\textwidth} 
    \includegraphics[width=\textwidth]{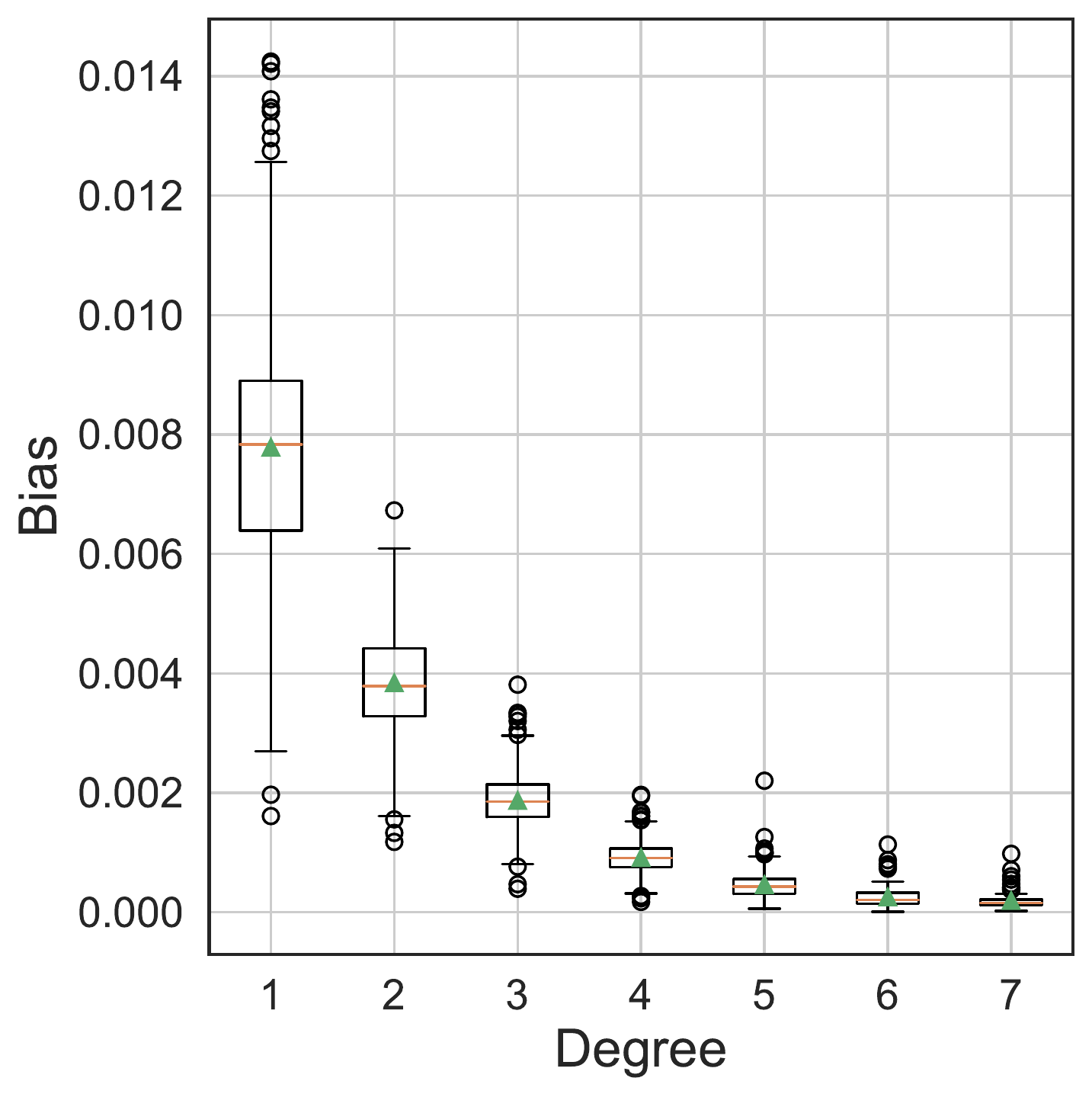}
    \caption{The bias \vs the node degree.}
    \label{fig:bpdegree}
  \end{minipage}%
\hspace{0.1cm}
  \begin{minipage}[t]{0.325\textwidth} 
    \includegraphics[width=\textwidth]{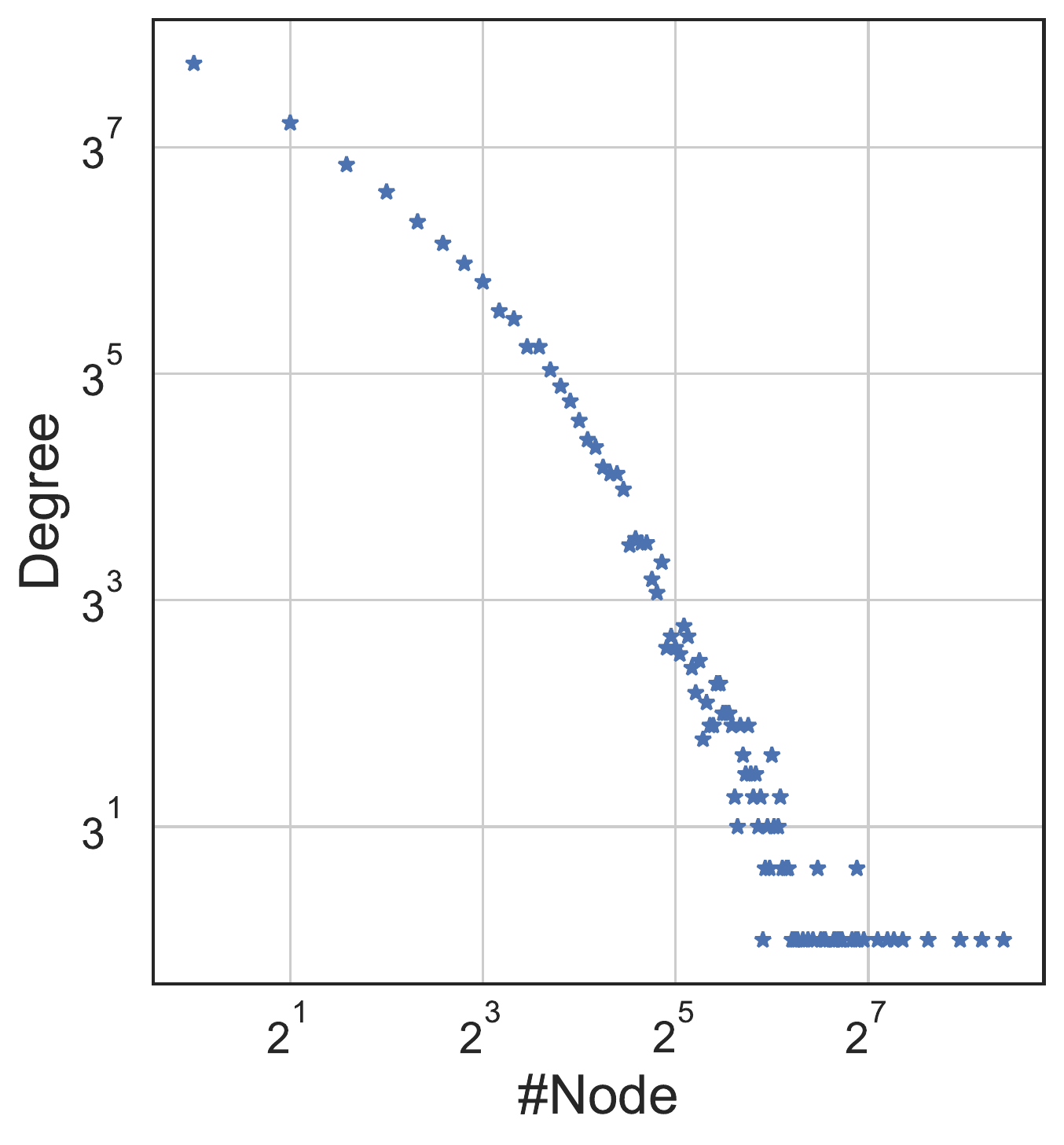}
    \caption{Node degrees obey the power law distribution.}
    \label{fig:powlaw}
  \end{minipage}%
\end{figure}

\begin{figure*}[t]
    \centering
    \begin{subfigure}[t]{0.27\textwidth}
         \centering
         \includegraphics[width=\textwidth]{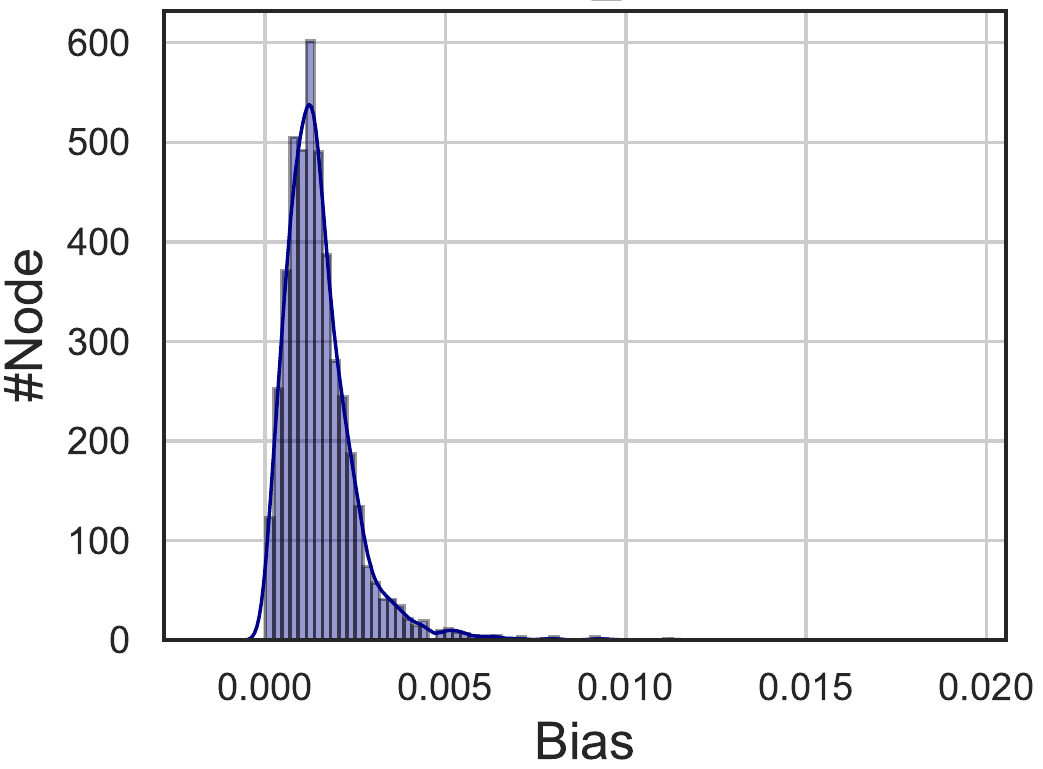}
         \caption{NN with AM ($nn\_a$)}
         \label{fig:disa}
     \end{subfigure}
     \begin{subfigure}[t]{0.23\textwidth}
         \centering
         \includegraphics[width=\textwidth]{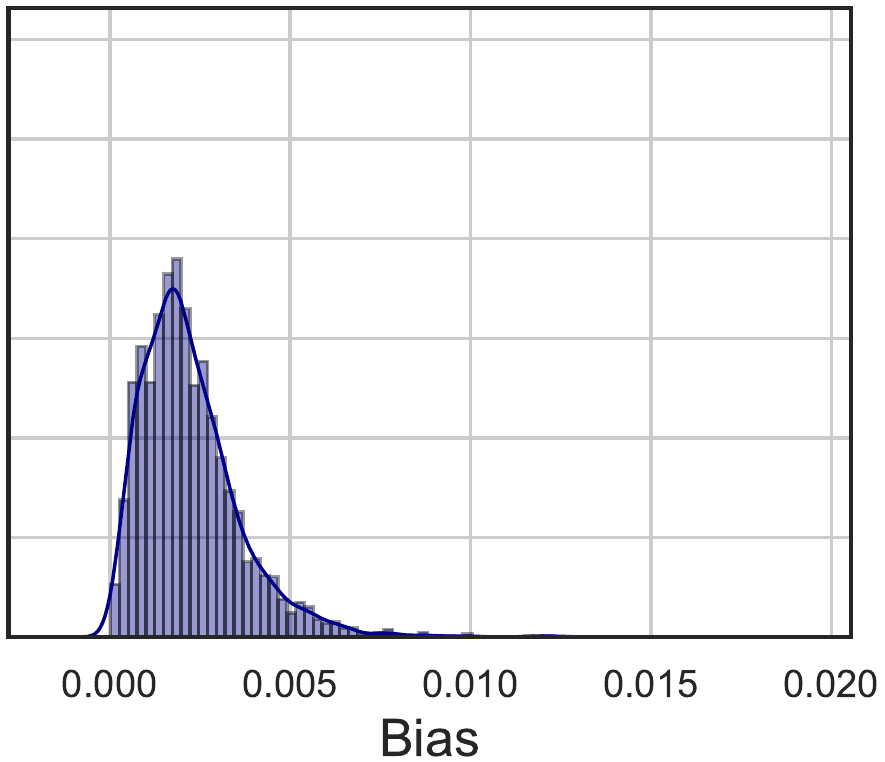}
         \caption{GNN with AM ($gnn\_a$)}
         \label{fig:disb}
     \end{subfigure}
     \begin{subfigure}[t]{0.23\textwidth}
         \centering
         \includegraphics[width=\textwidth]{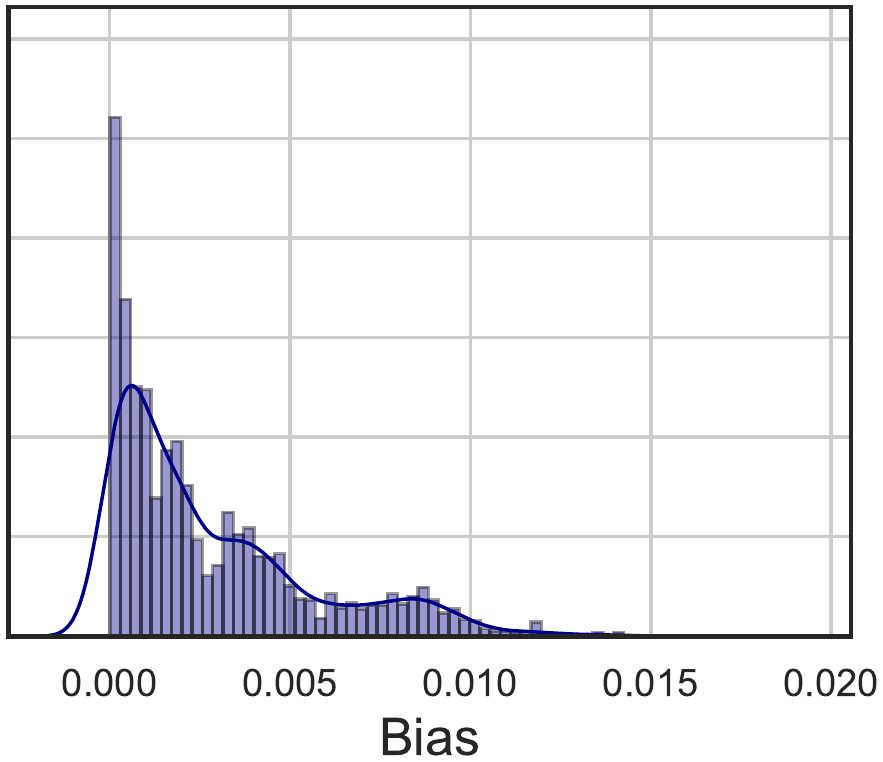}
         \caption{GNN with EP ($gnn\_ea$)}
         \label{fig:disc}
     \end{subfigure}
     \begin{subfigure}[t]{0.23\textwidth}
         \centering
         \includegraphics[width=\textwidth]{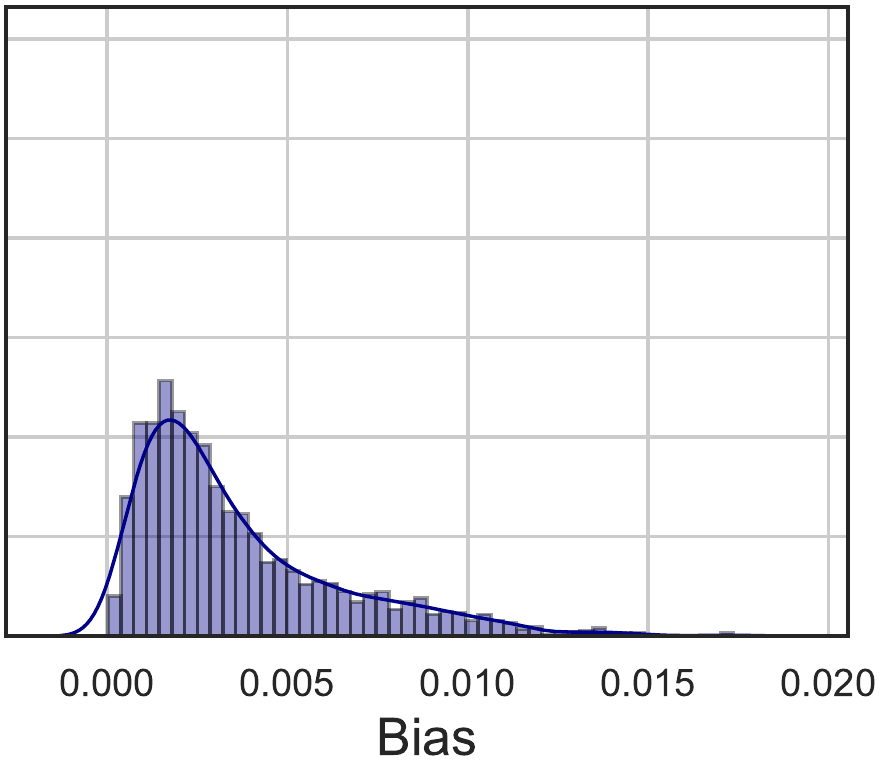}
         \caption{GNN with EP \& AM ($gnn\_ea)$}
         \label{fig:disd}
     \end{subfigure}

    \caption{The bias distribution of all nodes on the Cora dataset \wrt different augmentation strategies and encoder .
    }
    \label{fig:biasdistri}
\end{figure*}
\begin{table*}[t]
    \centering
    \caption{Results (Cora, CiteSeer, and WikiCS) given a common testbed for different feature augmentation strategies realized by Eq.~(\ref{eq:coverr}). The symbol $\bm{e}_i$ denotes the one-hot vector and $\mathcal{O}$ is the zero matrix.}
    \label{tab:diffaugmentation}
    \resizebox{\textwidth}{!}{
    \begin{tabular}{l|l|ll|ccc}
\toprule
\toprule
        \textbf{Feature Augmentation}  & \textbf{Type} & $\bm{P}\in\mathbb{R}^{k \times n}$ & $\bm{E}\in \mathbb{R}^{k \times d}$ & \textbf{Cora} & \textbf{CiteSeer} & \textbf{WikiCS} \\
        \midrule
        Gaussian Noise Injection & Stochastic &$\bm{P}=\bm{I}$ &  $\bm{E}\sim \mathcal{N} (0, 1)$ &$0.8271$ &$0.7134$ & $0.7823$\\
         SVD & Deterministic &$\bm{P} = \bm{U}^\top, \bm{X} = \bm{U} \bm{\Sigma} \bm{V}^\top$ &$\bm{E} = \mathcal{O}$ & $0.8269$ & $0.7142$ & $0.7814$\\
         Random Selection & Stochastic&$\bm{P}$ = $\{\bm{e}_{i} | \mathcal{P} (i)=\frac{\|\bm{X}_{i:}\|_2}{\|\bm{X}\|_F}\}$ & $\bm{E} = \mathcal{O}$ &$0.8245$ & $0.7121$ & $0.7811$\\
\rowcolor{LightCyan} Random Projection & Stochastic &
$\bm{P}\sim \mathcal{N} (0, 1)$
     & $\bm{E} = \mathcal{O}$ &$\bm{0.8425}$& $\bm{0.7247}$ &  $\bm{0.7911}$\\
    \bottomrule
    \bottomrule
    \end{tabular}
    }
\end{table*}
\subsection{The Bias Problem of Graph Augmentation with GNNs (RQ1)}
In Section~\ref{sec:motivation}, we intuitively point out the pitfall of the topology GA by providing an illustrative example. Based on that observation, we hypothesize that the  topology GA  introduces a substantial bias into the node embeddings used by the contrastive loss, which deteriorates the quality of features from the pre-training step, thus affecting downstream tasks. In this section, we conduct a quantitative analysis of this problem.

\vspace{0.1cm}
\noindent\textbf{Experimental Protocol.}
We adopt the edge perturbation and attribute masking, the most commonly used strategies for the graph augmentation~\cite{zhu2020deep, zhu2021graph, velickovic2019deep}. To show the difference between the attribute GA and the topology GA, we use MLP and GNN to encode the original features. For a fair comparison, MLP and GNN share the fixed weights. We only change the type of augmentations used to obtain the node embedding. 
Specifically, We denote GNN with the edge perturbation (EP) as $gnn\_e$, GNN with attribute masking (AM) as $gnn\_a$, GNN with AM and EP as $gnn\_ea$, MLP with AM as $nn\_a$.
For each node from Cora, we generate $500$ augmented samples in four variants to compute their bias by Eq.~(\ref{eq:bias}). 

\vspace{0.1cm}
\noindent\textbf{Bias Patterns of Graph Augmentations.}
Figure~\ref{fig:biasdistri} shows  distributions containing bias. It is obvious that $gnn\_ea$ contains more node embeddings with a larger bias compared to $nn\_ea$, judging by the 
distribution shift towards right. 
This confirms our hypothesis that the graph augmentation introduces the bias. Moreover, 
the long-tailed trend is mainly caused by the edge perturbation. Figure~\ref{fig:bpdegree} shows that nodes with a small degree  exhibit more bias, which confirms our insight 
that for node embeddings of low-degree nodes, obtained via GNN, removing any edges perturbs these embedding significantly. Real-world graphs (citation network, social network, \etc) follow the power-law distribution (shown in Figure~\ref{fig:powlaw}) meaning a few of nodes connect with the majority of the edges, whereas the majority of nodes have only a few of edges (low degree). For example, 54\%, 68\%, 71\%, 53\% nodes of Cora, CitSeer, PubMed, and DBLP datasets have less than 3 edges. Thus, applying the graph augmentation in 
contrastive learning results in a large bias.

\begin{table*}[t]
\centering
\caption{Node classification in terms of accuracy (\%) with standard deviation. The highest performance is highlighted in boldface. COSTA$_{MV}$ and COSTA$_{SV}$ denote the variants of multi-view and single-view setting respectively, OOM indicates Out-Of-Memory.}
\label{tab:mainresult1}
\resizebox{0.9\textwidth}{!}{
\begin{tabular}{lcccccc}
\toprule 
\textbf{Method} & \textbf{Training Data} & \textbf{Wiki-CS} & \textbf{Amazon-Computers} & \textbf{Amazon-Photo} & \textbf{Coauthor-CS} & \textbf{Coauthor-Physics} \\
\midrule Raw features & $\bm{X}$ & $71.98 \pm 0.00$ & $73.81 \pm 0.00$ & $78.53 \pm 0.00$ & $90.37 \pm 0.00$ & $93.58 \pm 0.00$ \\
Node2vec & $\bm{A}$ & $71.79 \pm 0.05$ & $84.39 \pm 0.08$ & $89.67 \pm 0.12$ & $85.08 \pm 0.03$ & $91.19 \pm 0.04$ \\
DeepWalk & $\bm{A}$ & $74.35 \pm 0.06$ & $85.68 \pm 0.06$ & $89.44 \pm 0.11$ & $84.61 \pm 0.22$ & $91.77 \pm 0.15$ \\
DeepWalk  & $\bm{X}, \bm{A}$ & $77.21 \pm 0.03$ & $86.28 \pm 0.07$ & $90.05 \pm 0.08$ & $87.70 \pm 0.04$ & $94.90 \pm 0.09$ \\
GAE     & $\bm{X}, \bm{A}$ & $70.15 \pm 0.01$ & $85.27 \pm 0.19$ & $91.62 \pm 0.13$ & $90.01 \pm 0.71$ & $94.92 \pm 0.07$ \\
VGAE    & $\bm{X}, \bm{A}$ & $75.63 \pm 0.19$ & $86.37 \pm 0.21$ & $92.20 \pm 0.11$ & $92.11 \pm 0.09$ & $94.52 \pm 0.00$ \\
DGI     & $\bm{X}, \bm{A}$ & $75.35 \pm 0.14$ & $83.95 \pm 0.47$ & $91.61 \pm 0.22$ & $92.15 \pm 0.63$ & $94.51 \pm 0.52$ \\
GMI     & $\bm{X}, \bm{A}$ & $74.85 \pm 0.08$ & $82.21 \pm 0.31$ & $90.68 \pm 0.17$ & OOM              & OOM              \\
MVGRL   & $\bm{X}, \bm{A}$ & $77.52 \pm 0.08$ & $87.52 \pm 0.11$ & $91.74 \pm 0.07$ & $92.11 \pm 0.12$ & $95.33 \pm 0.03$ \\
GRACE   & $\bm{X}, \bm{A}$ & $78.31 \pm 0.05$ & $87.80 \pm 0.23$ & $92.53 \pm 0.16$ & $\mathbf{92.95 \pm 0.03}$ & $95.72 \pm 0.03$ \\
GCA     & $\bm{X}, \bm{A}$ & $78.23 \pm 0.04$ & $87.54 \pm 0.49$ & $92.24 \pm 0.21$ & $92.95 \pm 0.13$ & $95.73 \pm 0.03$ \\
G-BT  & $\bm{X}, \bm{A}$ & $76.83 \pm 0.73$ & $87.93 \pm 0.36$ & $92.46 \pm 0.35$ & $92.91 \pm 0.25$ & $95.25 \pm 0.13$ \\
\midrule
\rowcolor{LightCyan} COSTA$_{SV}$ & $\bm{X}, \bm{A}$ & $79.03 \pm 0.05$ & $88.26 \pm 0.03$ & $92.30 \pm 0.25$ & $\mathbf{92.95} \pm 0.12$ & $\bm{95.74 \pm 0.02}$ \\
\rowcolor{LightCyan} COSTA$_{MV}$ & $\bm{X}, \bm{A}$ & $\mathbf{79.12 \pm 0.02}$ & $\mathbf{88.32 \pm 0.03}$ & $\mathbf{92.56 \pm 0.45}$ & $92.94 \pm 0.10$ & $95.60 \pm 0.02$ \\
\bottomrule
\end{tabular}
}
\end{table*}
\begin{figure*}[t] 
  \begin{minipage}[h]{0.35\textwidth} 
    \includegraphics[width=\textwidth]{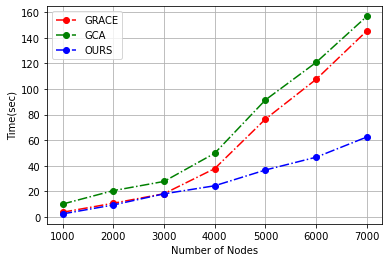}
    \caption{The runtime comparison of single-view COSTA (COSTA$_{SV}$) \vs MV-GCL models such as GCA and GRACE.}
    \label{fig:runningtime}
  \end{minipage}%
  \hspace{0.1cm}
 \begin{minipage}[h]{0.55\textwidth} 
    \resizebox{\textwidth}{!}{
\begin{tabular}{lccccc}
\toprule
\textbf{Method} & \textbf{Training Data} & \textbf{Cora} & \textbf{Citeseer} & \textbf{Pubmed} & \textbf{DBLP} \\
\midrule Raw features & $\bm{X}$ & $64.8$ & $64.6$ & $84.8$ & $71.6$ \\
node2vec & $\bm{A}$ & $74.8$ & $52.3$ & $80.3$ & $78.8$ \\
DeepWalk & $\bm{A}$ & $75.7$ & $50.5$ & $80.5$ & $75.9$ \\
DeepWalk  & $\bm{X}, \bm{A}$ & $73.1$ & $47.6$ & $83.7$ & $78.1$ \\
GAE & $\bm{X}, \bm{A}$ & $76.9$ & $60.6$ & $82.9$ & $81.2$ \\
VGAE & $\bm{X}, \bm{A}$ & $78.9$ & $61.2$ & $83.0$ & $81.7$ \\
DGI & $\bm{X}, \bm{A}$ & $82.6 \pm 0.4$ & $68.8 \pm 0.7$ & $86.0 \pm 0.1$ & $83.2 \pm 0.1$ \\
GRACE & $\bm{X}, \bm{A}$ & ${83.3} \pm {0.4}$ & ${72.1\pm 0.5}$ & $\bm{8 6 . 7 \pm 0 . 1}$ & ${8 4 .2}\pm{0.1}$ \\
GCA & $\bm{X}, \bm{A}$ & ${82.8} \pm {0.3}$ & ${71.5\pm 0.3}$ & ${8 6 . 0 \pm 0 . 2}$ & ${8 3 .1}\pm{0.2}$ \\
\midrule
\rowcolor{LightCyan} COSTA$_{SV}$ & $\bm{X}, \bm{A}$ & $84.3\pm0.3 $ & $72.8\pm 0.3$ & ${86.2\pm0.1}$ & ${84.3}\pm{0.1}$ \\
\rowcolor{LightCyan} COSTA$_{MV}$ & $\bm{X}, \bm{A}$ & $\bm{84.3 \pm 0.2}$ & $\bm{72.9\pm 0.3}$ & $86.0\pm0.2$ & $\bm{84.5\pm0.1}$ \\
\bottomrule
\end{tabular}
}
\captionof{table}{Results on Cora, CiteSeer, PubMed, and DBLP. We use the same setting as the setting used for experiments in  Table~\ref{tab:mainresult1}.}
\label{tab:mainresult2}
\end{minipage}
\end{figure*}
\subsection{Comparison with the State-of-the-Art Methods (RQ2)}
In this section, we compare COSTA to other baseline models to answer \textbf{RQ2}. We use the same experimental setup as the representative MV-GCL method (\ie, GCA~\cite{zhu2021graph}, and GRACE~\cite{zhu2020deep}) to perform a fair comparison to these methods. 
Unless stated otherwise, the random projection is employed as the default setting of COSTA as it balances well between accuracy and efficiency. A detailed comparison between different feature augmentations given different matrix sketching schemes is shown in Section~\ref{sec:otherfa}.

\vspace{0.1cm}
\noindent\textbf{Datasets.}
To evaluate our method, we adopt nine commonly used benchmark datasets in the previous works~\cite{zhu2020deep, zhu2021graph,velickovic2019deep}, including citation networks (Cora, CiteSeer, Pubmed, DBLP, Coauthor-CS, and Coauthor-Physics) and social networks (Wiki-CS, Amazon-Computers, Amazon-Photo)~\cite{kipf2016semi,sinha2015overview,mcauley2015image, mernyei2020wiki}. Detailed descriptions and statistics are given in Appendix~\ref{sec:dataset}. Apart from Wiki-CS adopting the public split, other datasets are randomly divided into 10\%, 10\%, 80\% for training, validation, and testing.

\vspace{0.1cm}
\noindent\textbf{Evaluation Protocol.} For each experiment, we adopt the same 
evaluation scheme as in~\cite{velickovic2019deep, zhu2020deep, zhu2021graph}, where each model is firstly trained in an unsupervised manner on the whole graph with node features. Then, we transform the raw features into the resulting embeddings with the use of the trained encoder. Next, we train an $\ell_2$-regularized logistic regression classifier from the Scikit-Learn library~\cite{pedregosa2011scikit} with the use of embeddings obtained in the previous step. We also perform a grid search over the regularization parameter with the following values $\{2^{-10}, 2^{-9}, \dots , 2^{-1}\}$. 
We compute the classification accuracy and report the mean and standard deviations for 20 model initializations and splits. 

\vspace{0.1cm}
\noindent\textbf{Baselines.} 
To compare COSTA with previous works, we choose the representative baselines from traditional graph self-supervised learning, autoencoder-based model, and contrastive-based graph self-supervised learning.  Methods include (i) Random walk based models: Deepwalk~\cite{perozzi2014deepwalk} and node2vec~\cite{grover2016node2vec}, (ii) Autoencoder Based models: GAE and VGAE~\cite{kipf2016variational}, (iii) the contrastive-based models including Deep Graph Infomax (DGI)~\cite{velickovic2019deep}, Graphical Mutual Information Maximization (GMI)~\cite{peng2020graph}, Graph Barlow Twins (G-BT)~\cite{bielak2021graph} and Multi-View Graph Representation Learning (MVGRL)~\cite{hassani2020contrastive}, GRACE~\cite{zhu2020deep} and GCA~\cite{zhu2021graph}. Note that all the contrastive models use the topology graph augmentation by default. 
For all baselines, we report their performance based on the official implementations and we use  use default hyper-parameters from original papers.

\vspace{0.1cm}
\noindent\textbf{Main Results.}
Tables~\ref{tab:mainresult1} and \ref{tab:mainresult2} 
show that COSTA achieves competitive performance compared to the baseline methods, and even surpasses them on most datasets. These results demonstrate that COSTA is an effective framework leveraging the advantage of feature augmentations. Specifically, the superiority of COSTA is confirmed by the fact that both single- and multi-view COSTA variants, COSTA$_{SV}$ and COSTA$_{MV}$, outperform several MV-GCL models that use the topology graph augmentation (\ie, GCA, GRACE, MVGRL) on several datasets (Cora, CiteSeer, DBLP, Wiki-CS, Amazon-Computers, AM-Photo and Coauthor-Physics) and achieve comparable results on PubMed, and Coauthor-CS datasets. We note that  datasets on which COSTA$_{SV}$ does not achieve SOTA have a small number of nodes with low node degrees (\ie, only around 10\% of nodes in Amazon-photo and Coauthor-Physics have the degree less than 3, meaning less bias is introduced by the topology GA. However, COSTA$_{SV}$ requires less runtime and memory  to achieve the comparable performance. This is attributed to the single-view design (SV-GCL) of COSTA$_{SV}$ (Section~\ref{sec:SVGCL}). We also note that COSTA$_{MV}$ typically outperforms COSTA$_{SV}$ by a small margin which suggests that single-view augmentation strategies are a good choice for GA.   In addition, we note that most of the contrastive learning models outperform models based on the reconstruction error (\ie, GAE, VGAE, DeepWalk, Node2Vec), which reflects the superiority of contrastive learning.

\vspace{0.1cm}
\noindent\textbf{Running time.} 
We measure the runtime to further validate the practicality of the single-view COSTA (COSTA$_{SV}$) in terms of time complexity. We mainly compare it to GCA and GRACE (Amazon-Computers dataset) because both GCA and GRACE are the representative models for the multi-view contrastive learning framework utilizing graph augmentations. Note that  GCA uses adaptive graph augmentations. To form  Figure~\ref{fig:runningtime}, we sampled a subgraph with a fixed number of nodes from $1,000$ to $8,000$. Figure~\ref{fig:runningtime} shows the training time for 1,000 epochs given  different numbers of nodes. The figure shows that our method is faster than the other two models. What stands out is that the gap between them becomes more apparent as the number of nodes increases. COSTA$_{SV}$ becomes 2 times faster than the other models at $\geq 5,000$ nodes. We attribute this to the single-view setting with the feature augmentation. Note that COSTA$_{SV}$  computes the node feature matrix once before feeding it into the projection head, and the feature augmentation 
effectively can be understood as reducing the number of nodes, as our feature augmentation acts on columns of hidden feature matrix. 
Thus, we form one relatively small similarity matrix for the contrastive loss. 
In contrast, the MV-GCL framework incurs a higher complexity. Our experiments confirm that COSTA$_{SV}$ is  efficient in practice. Moreover, COSTA$_{SV}$ can be accelerated by employing sketching by a sparse matrix without sacrificing its performance, as shown in Figure~\ref{fig:density} and Appendix~\ref{sec:vsrp}.

\subsection{Ablations/Performance Analysis (RQ3 \& 4)}
\begin{figure*}[t] 
  \begin{minipage}[c]{0.31\textwidth} 
    \includegraphics[width=\textwidth]{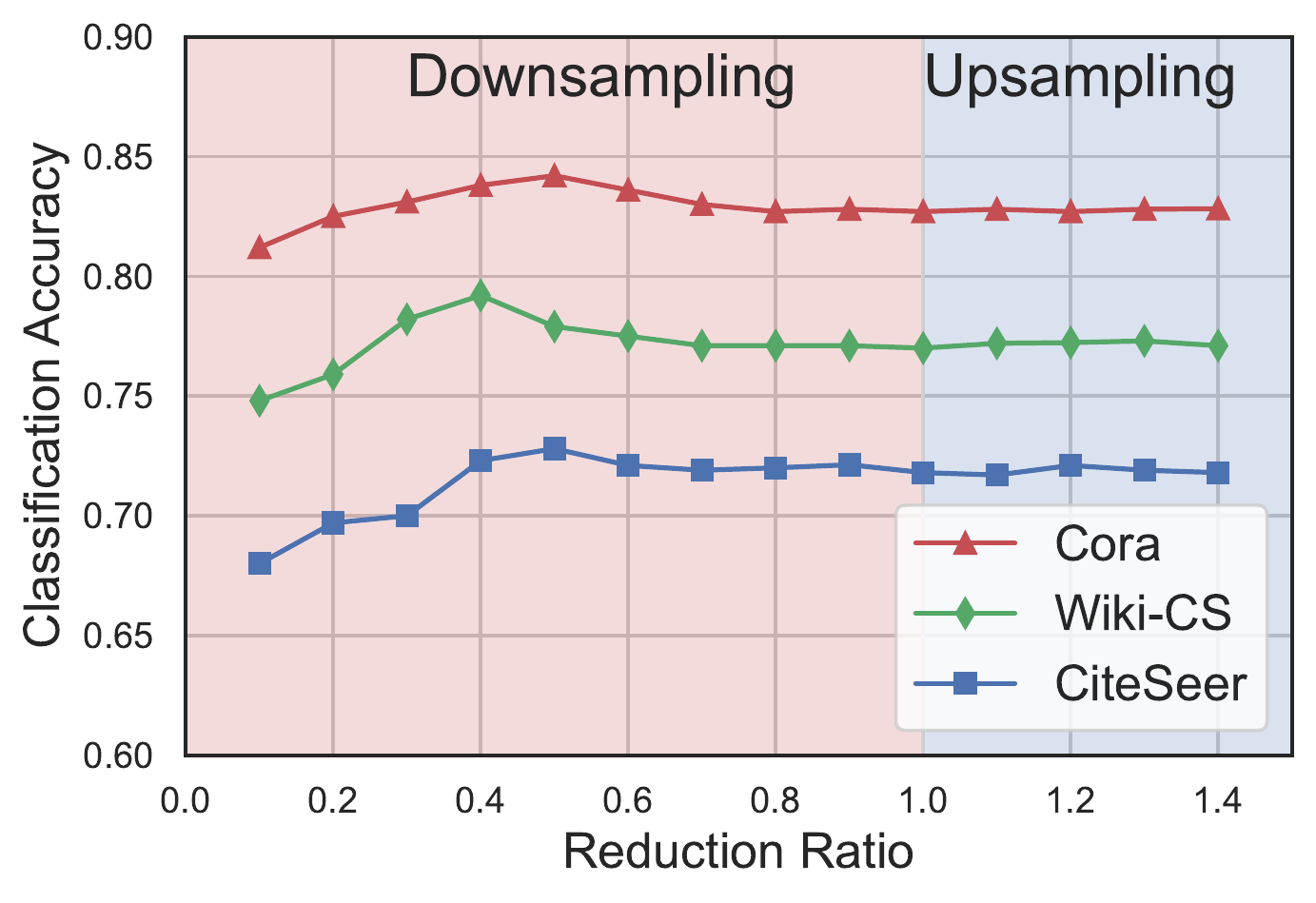}
    \caption{Accuracy \wrt sketching ratios.}
    \label{fig:reduction}
  \end{minipage}%
  \hspace{0.1cm}
  \begin{minipage}[c]{0.31\textwidth} 
   \includegraphics[width=\textwidth]{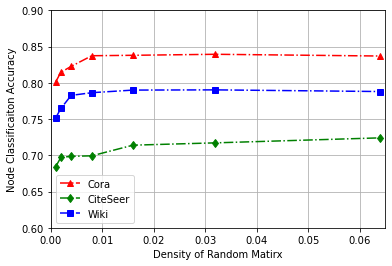}
    \caption{Accuracy \wrt  matrix densities.}
    \label{fig:density}
  \end{minipage}%
  \hspace{0.1cm}
  \begin{minipage}[c]{0.32\textwidth} 
     \includegraphics[width=\textwidth]{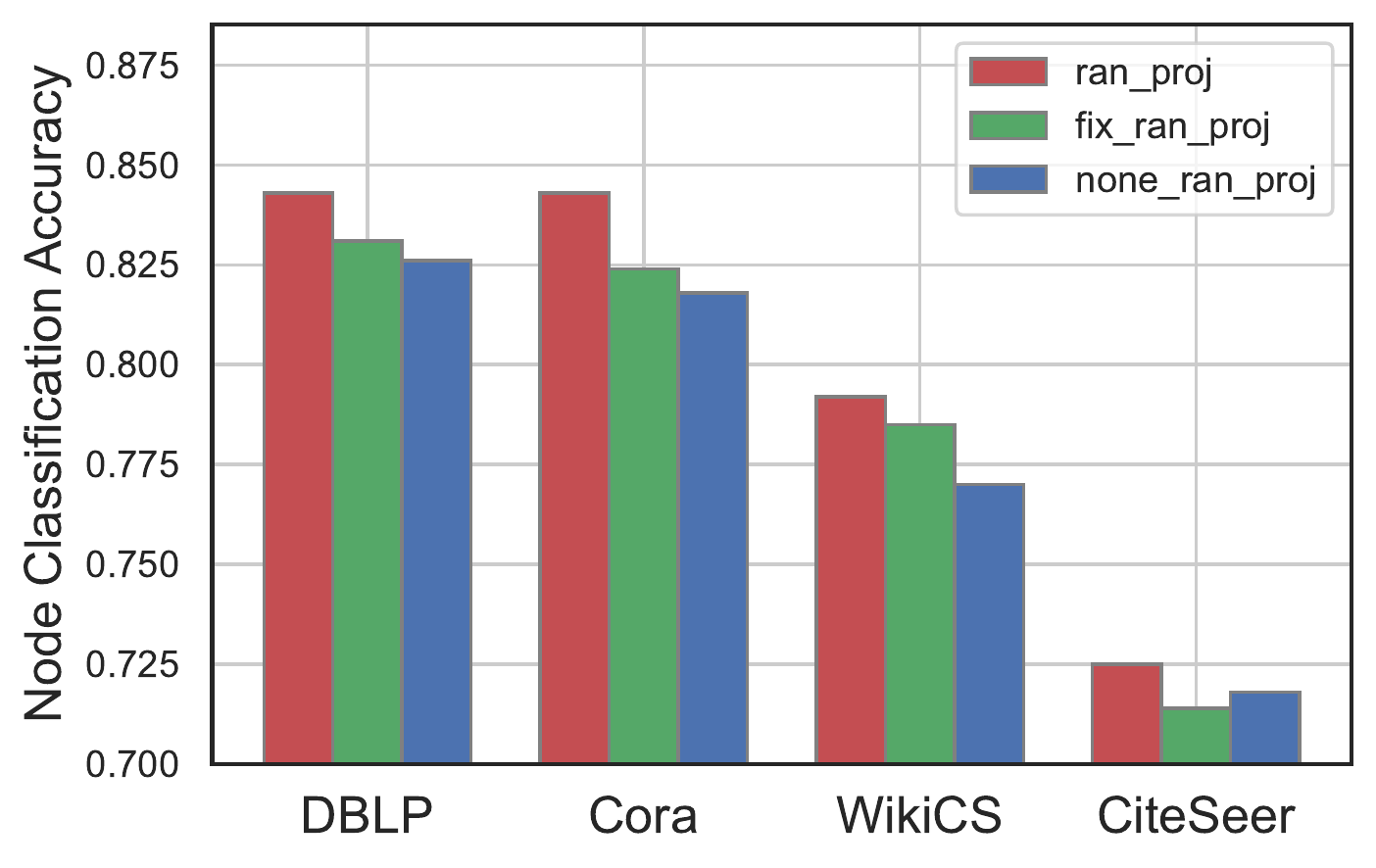}
    \caption{Ablation study on different variants of random projection.}
    \label{fig:ablation}
  \end{minipage}%
\end{figure*}
Below we use the single-view COSTA (COSTA$_{SV}$), as multi-view COSTA has a similar performance. 
We firstly show the superiority of 
random projection by comparing it with other matrix sketching variants. Subsequently, we ablate and discuss the factors that lead to the success of the random projection. Finally, we show the effect of the number of augmented samples that influence the error bound.

\vspace{0.1cm}
\noindent\textbf{Comparison with Other Feature Augmentations.}
\label{sec:otherfa}
We compare the random projection with other matrix sketching strategies. Table~\ref{tab:diffaugmentation} uses the common tested based on COSTA$_{SV}$: only the definition of $\bm{P}$ and $\bm{E}$ are varied. We note that the random projection  works consistently better than all the other strategies as the random projection introduces a lesser error compared to the random selection and random noise strategies. It is somewhat surprising that although the random projection is not the optimal solution to Eq.~(\ref{eq:coverr}), 
it still outperforms the SVD-based  sketching. Such a good performance comes from the following facts: (i) the error bound of random projection is sufficiently small to maintain a good sketch; and (ii) compared to SVD, which is deterministic, random projection adds stochastic perturbations to the model (variance is a source of feature augmentations), which serves as a regularization. 

\begin{table}[t]
\begin{center}
\caption{Ablation study: the topology graph augmentation  \vs the feature augmentation (single-view COSTA).}
\label{tab:AblationStudyOnAug}
\begin{tabular}{cc|ccc}
\toprule
Graph Aug. & Feature Aug. & \textbf{WikiCS} & \textbf{Cora} & \textbf{CiteSeer} \\
\midrule 
$\times$     &$\times$     &$74.31$ &$80.04$   &$71.35$\\
$\checkmark$ &$\times$     &$77.67$ &$82.13$   &$71.93$\\
$\times$     &$\checkmark$ &$79.03$ &$84.30$   &$72.81$\\
\rowcolor{LightCyan} $\checkmark$ &$\checkmark$ &$\mathbf{79.07}$ &$\mathbf{84.33}$   &$\mathbf{72.90}$\\
\bottomrule
\end{tabular}
\end{center}
\end{table}

\vspace{0.1cm}
\noindent\textbf{Ablations \wrt the Augmentation Type in GCL.}
Below, we investigate different augmentation types, \ie, the topology graph augmentation and the feature augmentation. To minimize other factors other than the augmentation strategy that might affect the results (\ie, multi-view setting), we opt for the single-view COSTA. 
We only replace the augmentation type. Table~\ref{tab:AblationStudyOnAug} shows that  without any augmentations, the model performs badly, showing the necessity of data augmentations. Furthermore, we observe that applying either the topology graph augmentation or the feature augmentation can improve results. However, the improvement of feature augmentation is larger compared to the topology graph augmentation, highlighting  the effectiveness of COSTA. 

\vspace{0.1cm}
\noindent\textbf{Ablations on the Random Projection.}
To see where the performance improvement comes from, we conduct an ablation study on the COSTA$_{SV}$ with the random projection. Firstly, we  fix the random projection so that it remains the same in each epoch, denoting this variant as $ran\_proj\_fix$, and then we eliminate the random projection completely ($none\_ran\_proj$). We experiment on Cora, CiteSeer, PubMed, and DBLP. Figure~\ref{fig:ablation} shows that fixing the random projection throughout the experiment causes the performance drop in all four datasets. In contrast, forming a new random projection matrix per epoch generates a variety of  feature augmentations obeying the variance bound of random projection. As projecting is performed along columns of the hidden feature matrix, this is an equivalent of drawing different new nodes (obeying the mean and variance) with each change of random projection. 
Notably, removing the random projection completely from training  decreases the accuracy on three datasets, as expected. 

\vspace{0.1cm}
\noindent\textbf{Downsampling \vs Upsampling the Node Dimension.}
\label{sec:downsampling}
According to the error bound derived in Eq.~(\ref{eq:rpbound}), the bound is related to the number of augmented samples $k$. The use of COSTA lets us control the number of augmented features by adjusting the number of rows of $\bm{P} \in \mathbb{R}^{k\times n}$ in Eq.~(\ref{eq:coverr}). Downsampling and upsampling are applied by setting $k<|\mathcal{V}|$ and $k>|\mathcal{V}|$ respectively. We denote the $\frac{k}{|\mathcal{V}|}$ as the reduction ratio. Figure~\ref{fig:reduction} shows the relation between the performance and the reduction ratio. Specifically, we obtain the best performance for downsampling (the sketched rows are fewer than the number of nodes), which also accelerates computations of the contrastive loss (smaller similarity/dissimilarity matrices).

%% file: Content/Conclusion.tex
\section{Conclusions}
We have quantitatively and qualitatively analyzed the problems stemming from the topology graph augmentation of current GCL methods, and we have shown that such a strategy suffers from the bias problem. To overcome this bias, we have proposed the feature augmentation framework COSTA. We theoretically proved that the quality of augmented features obtained via COSTA are guaranteed and COSTA accelerates the speed of GCL by working well in the single-view mode. Our results are equivalent or better  than results of the standard   contrastive multi-view graph augmentation models that rely on topology-based augmentations.

%% file: Content/Appendix.tex
\section{Error Bounds of Different Feature Augmentations Preserving Second-order Statistics}
\noindent\textbf{Proof of Lemma~\ref{lemma:svd} (Matrix Sketching via SVD).}
\label{proof:svd}
\begin{proof}
According to the Eckart-Young-Mirsky theorem~\cite{golub1987generalization},  $\bm{X}_k = \bm{P}^\top \bm{P} \bm{X}$ is  the best $k$-rank approximation of $\bm{X}$ and  $\|\bm{X}-\bm{X}_k\|_2 = \sigma_{k+1}$. Thus, we have:
\begin{equation*}
\begin{aligned}
    \|\bm{X}^\top \bm{X} - \tilde{\bm{X}}^\top\tilde{\bm{X}})\|_2 &= \|\bm{X}^\top \bm{X} - \bm{X}^\top \bm{P}^\top \bm{P}\bm{X}\|_2, \\
    & = \|\bm{X}^\top(\bm{X} - \bm{P}^\top\bm{P}\bm{X})\|_2 = \|\bm{X}^\top(\bm{X} - \bm{X}_k)\|_2,\\
    & \leq \|\bm{X}^\top\|_2\| \bm{X} - \bm{X}_k\|_2 = \|\bm{X}\|_2^2 \frac{\|\bm{X} - \bm{X}_k\|_2}{\|\bm{X}\|_2}, \\
    &\leq \frac{\sigma_{k+1}}{\sigma_1} \|\bm{X}\|^2_F = \frac{\sigma_{k+1}}{\sigma_1} \operatorname{Tr} (\bm{X}^\top\bm{X}).
\end{aligned}
\end{equation*}
\end{proof}
\noindent\textbf{Proof of Lemma~\ref{lemma:randomselect} (Random Row Selection).}
\label{proof:randomselect}
To prove this Lemma, we use the following theorem from~\cite{drineas2006fast}.
\begin{theorem}
Let $\bm{A} \in \mathbb{R}^{d \times n}, \bm{B} \in \mathbb{R}^{r \times d}$ and $k \in \mathbb{Z}^{+}$such that $1 \leq k \leq n$ and $\left\{p_{i}\right\}_{i=1}^{n}$ be probability distribution over rows of $\bm{A}$ and columns of $\bm{B}$ such that $p_{i} \geq \frac{\beta \|\bm{A}_{i:}\|_2 \| \bm{B}_{i:} \|_2}{\sum_{j=1}^{n}\| \bm{A}_{j:}\|_2 \|\bm{A}_{j:} \|_2}$ for some positive constant $\beta \leq 1$. If matrix $C \in \mathbb{R}^{d \times k}$ is constructed by sampling columns of $A$ according to $\left\{p_{i}\right\}_{i=1}^{n}$ and matrix $\bm{D} \in \mathbb{R}^{k \times r}$ is constructed by picking same rows of $\bm{B}$, then with probability at least $1-\delta$:
$$
\|\bm{AB}-\bm{CD}\|_{F}^{2} \leq \frac{\mu^{2}}{\beta c}\|\bm{A}\|_{F}^{2}\|\bm{B}\|_{F}^{2},
$$
where $\delta \in(0,1), \mu=1+\sqrt{(8 / \beta) \log (1 / \delta)}$.
\label{theorem:1}
\end{theorem}

We set $\bm{A} = \bm{X}^\top \in \mathbb{R}^{d\times n}$, $\bm{B} = \bm{X} \in \mathbb{R}^{n \times d}$, $\bm{C} = \tilde{\bm{X}}^\top \in \mathbb{R}^{d \times k}$ and $\bm{D} = \tilde{\bm{X}} \in \mathbb{R}^{k \times d}$. Note that the distribution $p_{i} = \frac{\beta \|\bm{A}_{i:}\|_2 \| \bm{B}_{i:} \|_2}{\sum_{j=1}^{n}\|\bm{A}_{j:}\|_2 \|\bm{A}_{j:} \|_2} = \frac{ \|\bm{X}_{i:}\|}{\|\bm{X}\|_F}$ holds with $\beta = 1$. Using theorem~\ref{theorem:1}, we obtain bound:
\begin{equation}
    \begin{aligned}
    &\|\bm{X}^\top \bm{X} - \tilde{\bm{X}}^\top \tilde{\bm{X}}\|_F^2 \leq \frac{\mu^{2}}{k}\|\bm{A}\|_{F}^{4}, \\
    \Rightarrow &\|\bm{X}^\top \bm{X} - \tilde{\bm{X}}^\top \tilde{\bm{X}}\|_F \leq \frac{\mu}{\sqrt{k}}\|\bm{A}\|_{F}^{2}  \text{   (As $\|\bm{X}\|_2 \leq \|\bm{X}\|_F$)},\\
    \Rightarrow & \|\bm{X}^\top \bm{X} - \tilde{\bm{X}}^\top \tilde{\bm{X}}\|_2 \leq \frac{\mu}{\sqrt{k}}\operatorname{Tr}(\bm{X}^\top\bm{X}),
    \end{aligned}
\end{equation}
where $\delta \in(0,1), \mu=1+\sqrt{8\log (1 / \delta)}$. By setting $\frac{\mu}{\sqrt{k}} = \varepsilon$, we have:
\begin{equation}
    \mathcal{P}(\|\bm{X}^\top \bm{X} - \tilde{\bm{X}}^\top \tilde{\bm{X}}\|_F \leq \varepsilon \operatorname{Tr}(\bm{X}^\top\bm{X})) \geq 1 - e^{\left( -\frac{(\varepsilon \sqrt{k}-1)^2}{8}\right)},
\end{equation}
which completes the proof.

\noindent\textbf{Proof of Lemma~\ref{lemma:rp} (Random Projection).}
\label{proof:rp}
We prove the lemma by showing the following inequality  holds for any $\bm{x} \in \mathbb{R}^{n}$:
\begin{equation}
\label{eq:normpre}
    \mathcal{P}\left ( (1-\varepsilon)\|\bm{x}\|_2^{2} \leq\left\|\frac{1}{\sqrt{k}} \bm{P} \bm{x}\right\|^{2} \leq (1+\varepsilon)\|\bm{x}\|_2^{2}\right) \geq 1-e^{ \left(-\frac{k\varepsilon^2}{8}\right)}.
\end{equation}
We firstly show that  $\mathbb{E}\left[\left\|\frac{1}{\sqrt{k}} \bm{A} \bm{x}\right\|^{2}\right]=\|\bm{x}\|_2^2$:
\begin{equation}
\begin{aligned}
          & \mathbb{E}\left[\frac{1}{\sqrt{k}} \bm{P}\bm{x}\right] = \mathbb{E}\left[\sum_i^k \frac{1}{k}\sum_{j}(P_{ij} x_j)^2\right] = \mathbb{E}\left[\sum_i^k\frac{1}{k}\sum_{j,j'}(P_{ij} P_{ij'} x_j x_{j'})\right]\\
          & = \sum_i^k\frac{1}{k}\mathbb{E}\left[\sum_j \bm{P}^2_{jj}\bm{x}^2_j\right] =\sum_i^k\frac{1}{k} \sum_j x^2_j\\
          & = \|\bm{x}\|^2_2.
\end{aligned}
\end{equation}
This implies $\bm{X}_j \sim \mathcal{N}(0, 1))$ where $X_j = \frac{\bm{A}_{j:}\bm{x}}{\|\bm{x}\|_2}$. Then we obtain:
\begin{equation}
\begin{aligned}
& \mathcal{P}(\| \frac{1}{\sqrt{k}}\|\bm{P}\bm{x}\|_2^2 > (1+\varepsilon)\|\bm{x}\|_2)  = \mathcal{P}(\sum_{j=1}^k X^2_j  > (1+\varepsilon)k),\\
&=\mathcal{P}\left(e^{t \sum_{i=1}^{k} X_{j}^{2}}>e^{t(1+\varepsilon) k}\right)\\
& \leq \frac{\mathbb{E}\left[e^{t \sum_{j=1}^k X_{j}^{2}}\right]}{e^{t(1+\varepsilon) k}} \quad\text{  (apply Markov’s inequality)}\\ 
&= \frac{\prod_{j=1}^{k} \mathbb{E}\left[e^{t X_{j}^{2}}\right]}{e^{t(1+\varepsilon) k}} \quad\text{      (as $X_j$ is i.i.d)}\\ 
&= \frac{(\mathbb{E}\left[e^{t X_{j}^{2}}\right])^k}{e^{t(1+\varepsilon) k}}  \text{      (for $1<t<1/2$, it holds that $\mathbb{E}\left[e^{t X_{i}^{2}}\right] \leq\left(\frac{1}{\sqrt{1-2 t}}\right)$)}\\
& \leq \left(\frac{1}{\sqrt{1-2 t}}\right)^{k}\left(\frac{1}{e^{t(1+\varepsilon)}}\right)^{k}\\
& =\left(e^{-\varepsilon+\ln (1+\varepsilon)}\right)^{k / 2} \quad\text{ (using $\ln (1+\varepsilon) \leq \varepsilon-\frac{\varepsilon^2}{4}$)}\\
& \leq e^{-\varepsilon^2k/8}.
\end{aligned}
\end{equation}
Thus, we have:
\begin{equation}
    \mathcal{P}(\|\ \frac{1}{\sqrt{k}}\|\bm{P}\bm{x}\|_2^2 \leq (1+\varepsilon)\|\bm{x}\|^2_2)) \leq 1 - e^{-\varepsilon^2k/8}.
\end{equation}
In the similar way, it is easy to prove that:
\begin{equation}
    \mathcal{P}(\|\ \frac{1}{\sqrt{k}}\|\bm{P}\bm{x}\|_2^2 \geq (1-\varepsilon)\|\bm{x}\|^2_2)) \leq 1 - e^{-\varepsilon^2k/8}.
\end{equation}
Thus, it holds that:
\begin{equation}
\mathcal{P}((1-\varepsilon)\|\bm{x}\|^2_2 \leq \frac{1}{\sqrt{k}}\|\bm{P}\bm{x}\|_2^2 \leq (1+\varepsilon)\|\bm{x}\|^2_2) \geq 1 - e^{-\varepsilon^2k/8}.
\end{equation}

Suppose $\bm{X}^\top\bm{X} - \tilde{\bm{X}}^\top\tilde{\bm{X}} \succeq 0$ and $\bm{x}$ is the eigenvector of $\bm{X}^\top\bm{X} - \tilde{\bm{X}}^\top\tilde{\bm{X}}$ corresponding to its largest eigenvalue $\sigma_{\max}$, then we have:
\begin{align}
    \|\bm{X}^\top\bm{X} - \tilde{\bm{X}}^\top\tilde{\bm{X}}\|_2 &= \bm{x}^\top (\bm{X}^\top\bm{X} - \tilde{\bm{X}}^\top\tilde{\bm{X}}) \bm{x} = \sigma_{\max} \\
    & = \bm{x}^T \bm{X}^T\bm{X} \bm{x} - \bm{x}^T \tilde{\bm{X}}^T\tilde{\bm{X}} \bm{x} \\
    & = \|\bm{Xx}\|_2^2 - \|\tilde{\bm{X}}\bm{x}\|_2^2 \\
    & = \|\bm{Xx}\|_2^2 - \|\bm{PXx}\|_2^2 .
\end{align}
Applying Eq.~(\ref{eq:normpre}), there is at least probability $1 - e^{-\varepsilon^2k/8}$ such that:
\begin{equation}
   \begin{aligned}
    \|\bm{X}^\top\bm{X} - \tilde{\bm{X}}^\top\tilde{\bm{X}}\|_2 & \leq \|\bm{Xx}\|_2^2 + (\varepsilon-1) \|\bm{Xx}\|_2^2\\
   & = \varepsilon \|\bm{Xx}\|^2_2 \\
   &\leq \varepsilon \|\bm{X}\|^2_2\|\bm{x}\|^2_2 \text{  $(\|\bm{x}\|^2_2 = 1)$}\\ 
   & \leq \varepsilon \operatorname{Tr}(\bm{X}^\top\bm{X})
\end{aligned} 
\end{equation}
In similar way, it is easy to show that above equation also holds for $\tilde{\bm{X}}^\top\tilde{\bm{X}} - \bm{X}^\top\bm{X} \succeq 0$, Thus, we have:
\begin{equation}
    \mathcal{P}(\|\bm{X}^\top\bm{X} - \tilde{\bm{X}}^\top\tilde{\bm{X}}\|_2 \leq \varepsilon \operatorname{Tr}(\bm{X}^\top\bm{X})) \geq  1 - e^{-\varepsilon^2k/8},
\end{equation}
which completes the proof.

\section{Accelerating the Feature Augmentation with a Very Sparse Random Projection}
To accelerate the random projection, approach \cite{li2006very} presented a sparse random projection as an improvement over the Gaussian random projection, in which entries of $\bm{P}$ are  i.i.d. sampled from:
\begin{equation}
    p_{ij}=\left\{\begin{array}{rl}
\sqrt{s} &\text { with probability } \frac{1}{2s}, \\
0 &\text { with probability } 1-\frac{1}{s}, \\
-\sqrt{s} &\text { with probability } \frac{1}{2s},
\end{array}\right.
\label{eq:sij}
\end{equation}
where $\frac{1}{s}$ denotes the density of matirx $\bm{P}$.
\label{sec:vsrp}

\vspace{0.2cm}
The computation cost of the random projection is related to the sparsity of the sparse random matrix (SRP). The low density of the SRP  reduces the computational cost while it may affect the performance at the same time. To explore the relationship between the density and performance, we vary the density of the random matrix in the range of $[0.001, 0.002, 0.004, 0.008, 0.016, 0.32, 0.064] $. We plot the relationship between the performance and density on three datasets in Figure~\ref{fig:density}. It is apparent that the performance improves  as the density increases. However, the performance reaches a relatively high peak at a low density ($\frac{1}{s} < 0.01$). This suggests that one could enjoy the efficiency provided by the SPMM without sacrificing the performance.

\section{Statistics of Datasets}
\label{sec:dataset}

\begin{table}[!h]
    \centering
    \begin{tabular}{lcccc}
\toprule
\toprule
Dataset & \#Nodes & \#Edges & \#Features & \#Classes \\
\midrule Wiki-CS~\cite{mernyei2020wiki}  & 11,701 & 216,123 & 300 & 10 \\
Amazon-Computers\cite{mcauley2015image} & 13,752 & 245,861 & 767 & 10 \\
Amazon-Photo\cite{mcauley2015image}  & 7,650 & 119,081 & 745 & 8 \\
Coauthor-CS~\cite{sinha2015overview} & 18,333 & 81,894 & 6,805 & 15 \\
Coauthor-Physics~\cite{sinha2015overview} & 34,493 & 247,962 & 8,415 & 5 \\
Cora~\cite{kipf2016semi} & 2,708 & 5,429 & 1,433 & 7\\
Citeseer~\cite{kipf2016semi} & 3,327 & 4,732 & 3,703 & 6 \\
Pubmed~\cite{kipf2016semi} & 19,717 & 44,338 & 500 & 3 \\
DBLP~\cite{kipf2016semi} & 17,716 & 105,734 & 1,639 & 4 \\
\bottomrule
\bottomrule
\end{tabular}
\caption{Statistics of datasets used in our experiments.}
\label{tab:my_label}
\end{table}

Below we describe  datasets from Table \ref{tab:my_label}:
\begin{itemize}
    \item
    \textbf{Cora, CiteSeer, Pubmeb, DBLP.} These are well-known citation network datasets, in which nodes represent publications and edges indicate their citations. All nodes are labeled according to paper subjects~\cite{kipf2016semi}.
    \item 
    
    \textbf{WikiCS.} It is a network of  Wikipedia pages related to the computer science, with edges showing cross-references. Each article is assigned to one of 10 subfields (classes), with characteristics computed using the content's averaged GloVe embeddings. We make no changes to the 20 train/val/test data splits provided by~\cite{mernyei2020wiki}.
    \item 
    
    \textbf{Am-Computer, AM-Photo.} Both of these networks are based on Amazon's co-purchase data. Nodes represent products, while edges show how frequently they were purchased together. Each product is described using a Bag-of-Words representation based on the reviews (node features). There are ten node classes (product categories) and eight node classes (product categories), respectively~\cite{mcauley2015image}.
    \item 
    
    \textbf{Coauthor-CS, Coauthor-Physics.} These are two networks that were extracted from the Microsoft Academic Graph dataset. The edges reflect a collaboration between two authors, while the nodes represent writers. The keywords that each author uses in their articles are utilized to categorize them (Bag-of-Words representation; node features). According to~\cite{sinha2015overview}, there are 15 author research fields (node classes) and 5 author research fields (node classes).
\end{itemize}

%% file: kdd22.bbl
\begin{thebibliography}{10}

\bibitem{bachman2019learning}
Philip Bachman, R~Devon Hjelm, and William Buchwalter.
\newblock Learning representations by maximizing mutual information across
  views.
\newblock {\em arXiv preprint arXiv:1906.00910}, 2019.

\bibitem{bielak2021graph}
Piotr Bielak, Tomasz Kajdanowicz, and Nitesh~V Chawla.
\newblock Graph barlow twins: A self-supervised representation learning
  framework for graphs.
\newblock {\em arXiv preprint arXiv:2106.02466}, 2021.

\bibitem{chen2020simple}
Ting Chen, Simon Kornblith, Mohammad Norouzi, and Geoffrey Hinton.
\newblock A simple framework for contrastive learning of visual
  representations.
\newblock In {\em International conference on machine learning}, pages
  1597--1607. PMLR, 2020.

\bibitem{chen2021modeling}
Yankai Chen, Menglin Yang, Yingxue Zhang, Mengchen Zhao, Ziqiao Meng, Jianye
  Hao, and Irwin King.
\newblock Modeling scale-free graphs with hyperbolic geometry for
  knowledge-aware recommendation.
\newblock In {\em {WSDM} '22: The Fifteenth {ACM} International Conference on
  Web Search and Data Mining}. {ACM}, 2022.

\bibitem{chen2021attentive}
Yankai Chen, Yaming Yang, Yujing Wang, Jing Bai, Xiangchen Song, and Irwin
  King.
\newblock Attentive knowledge-aware graph convolutional networks with
  collaborative guidance for personalized recommendation.
\newblock In {\em The 38th IEEE International Conference on Data Engineering},
  2022.

\bibitem{devries2017dataset}
Terrance DeVries and Graham~W Taylor.
\newblock Dataset augmentation in feature space.
\newblock {\em arXiv preprint arXiv:1702.05538}, 2017.

\bibitem{drineas2006fast}
Petros Drineas, Ravi Kannan, and Michael~W Mahoney.
\newblock Fast monte carlo algorithms for matrices i: Approximating matrix
  multiplication.
\newblock {\em SIAM Journal on Computing}, 36(1):132--157, 2006.

\bibitem{feng2021survey}
Steven~Y Feng, Varun Gangal, Jason Wei, Sarath Chandar, Soroush Vosoughi,
  Teruko Mitamura, and Eduard Hovy.
\newblock A survey of data augmentation approaches for nlp.
\newblock {\em arXiv preprint arXiv:2105.03075}, 2021.

\bibitem{gao2021simcse}
Tianyu Gao, Xingcheng Yao, and Danqi Chen.
\newblock Simcse: Simple contrastive learning of sentence embeddings.
\newblock {\em arXiv preprint arXiv:2104.08821}, 2021.

\bibitem{golub1987generalization}
Gene~H Golub, Alan Hoffman, and Gilbert~W Stewart.
\newblock A generalization of the eckart-young-mirsky matrix approximation
  theorem.
\newblock {\em Linear Algebra and its applications}, 88:317--327, 1987.

\bibitem{grover2016node2vec}
Aditya Grover and Jure Leskovec.
\newblock node2vec: Scalable feature learning for networks.
\newblock In {\em Proceedings of the 22nd ACM SIGKDD international conference
  on Knowledge discovery and data mining}, pages 855--864, 2016.

\bibitem{hafidi2020graphcl}
Hakim Hafidi, Mounir Ghogho, Philippe Ciblat, and Ananthram Swami.
\newblock Graphcl: Contrastive self-supervised learning of graph
  representations.
\newblock {\em arXiv preprint arXiv:2007.08025}, 2020.

\bibitem{hamilton2017inductive}
William~L Hamilton, Rex Ying, and Jure Leskovec.
\newblock Inductive representation learning on large graphs.
\newblock In {\em Proceedings of the 31st International Conference on Neural
  Information Processing Systems}, pages 1025--1035, 2017.

\bibitem{hariharan2017low}
Bharath Hariharan and Ross Girshick.
\newblock Low-shot visual recognition by shrinking and hallucinating features.
\newblock In {\em Proceedings of the IEEE International Conference on Computer
  Vision}, pages 3018--3027, 2017.

\bibitem{hassani2020contrastive}
Kaveh Hassani and Amir~Hosein Khasahmadi.
\newblock Contrastive multi-view representation learning on graphs.
\newblock In {\em International Conference on Machine Learning}, pages
  4116--4126. PMLR, 2020.

\bibitem{he2020momentum}
Kaiming He, Haoqi Fan, Yuxin Wu, Saining Xie, and Ross Girshick.
\newblock Momentum contrast for unsupervised visual representation learning.
\newblock In {\em Proceedings of the IEEE/CVF Conference on Computer Vision and
  Pattern Recognition}, pages 9729--9738, 2020.

\bibitem{kipf2016semi}
Thomas~N Kipf and Max Welling.
\newblock Semi-supervised classification with graph convolutional networks.
\newblock {\em arXiv preprint arXiv:1609.02907}, 2016.

\bibitem{kipf2016variational}
Thomas~N Kipf and Max Welling.
\newblock Variational graph auto-encoders.
\newblock {\em arXiv preprint arXiv:1611.07308}, 2016.

\bibitem{deeper_look2}
Piotr Koniusz and Hongguang Zhang.
\newblock Power normalizations in fine-grained image, few-shot image and graph
  classification.
\newblock In {\em IEEE Transactions on Pattern Analysis and Machine
  Intelligence}, 2020.

\bibitem{li2006very}
Ping Li, Trevor~J Hastie, and Kenneth~W Church.
\newblock Very sparse random projections.
\newblock In {\em Proceedings of the 12th ACM SIGKDD international conference
  on Knowledge discovery and data mining}, pages 287--296, 2006.

\bibitem{liberty2013simple}
Edo Liberty.
\newblock Simple and deterministic matrix sketching.
\newblock In {\em Proceedings of the 19th ACM SIGKDD international conference
  on Knowledge discovery and data mining}, pages 581--588, 2013.

\bibitem{mcauley2015image}
Julian McAuley, Christopher Targett, Qinfeng Shi, and Anton Van Den~Hengel.
\newblock Image-based recommendations on styles and substitutes.
\newblock In {\em Proceedings of the 38th international ACM SIGIR conference on
  research and development in information retrieval}, pages 43--52, 2015.

\bibitem{mernyei2020wiki}
P{\'e}ter Mernyei and C{\u{a}}t{\u{a}}lina Cangea.
\newblock Wiki-cs: A wikipedia-based benchmark for graph neural networks.
\newblock {\em arXiv preprint arXiv:2007.02901}, 2020.

\bibitem{page1999pagerank}
Lawrence Page, Sergey Brin, Rajeev Motwani, and Terry Winograd.
\newblock The pagerank citation ranking: Bringing order to the web.
\newblock Technical report, 1999.

\bibitem{pedregosa2011scikit}
Fabian Pedregosa, Ga{\"e}l Varoquaux, Alexandre Gramfort, Vincent Michel,
  Bertrand Thirion, Olivier Grisel, Mathieu Blondel, Peter Prettenhofer, Ron
  Weiss, Vincent Dubourg, et~al.
\newblock Scikit-learn: Machine learning in python.
\newblock {\em the Journal of machine Learning research}, 12:2825--2830, 2011.

\bibitem{peng2020graph}
Zhen Peng, Wenbing Huang, Minnan Luo, Qinghua Zheng, Yu~Rong, Tingyang Xu, and
  Junzhou Huang.
\newblock Graph representation learning via graphical mutual information
  maximization.
\newblock In {\em Proceedings of The Web Conference 2020}, 2020.

\bibitem{perozzi2014deepwalk}
Bryan Perozzi, Rami Al-Rfou, and Steven Skiena.
\newblock Deepwalk: Online learning of social representations.
\newblock In {\em Proceedings of the 20th ACM SIGKDD international conference
  on Knowledge discovery and data mining}, pages 701--710, 2014.

\bibitem{shorten2019survey}
Connor Shorten and Taghi~M Khoshgoftaar.
\newblock A survey on image data augmentation for deep learning.
\newblock {\em Journal of Big Data}, 6(1):1--48, 2019.

\bibitem{sinha2015overview}
Arnab Sinha, Zhihong Shen, Yang Song, Hao Ma, Darrin Eide, Bo-June Hsu, and
  Kuansan Wang.
\newblock An overview of microsoft academic service (mas) and applications.
\newblock In {\em Proceedings of the 24th international conference on world
  wide web}, pages 243--246, 2015.

\bibitem{DBLP:conf/cikm/SongMZK21}
Zixing Song, Ziqiao Meng, Yifei Zhang, and Irwin King.
\newblock Semi-supervised multi-label learning for graph-structured data.
\newblock In {\em {CIKM}}, pages 1723--1733. {ACM}, 2021.

\bibitem{9737635}
Zixing Song, Xiangli Yang, Zenglin Xu, and Irwin King.
\newblock Graph-based semi-supervised learning: A comprehensive review.
\newblock {\em IEEE Transactions on Neural Networks and Learning Systems},
  pages 1--21, 2022.

\bibitem{uai_ke}
Ke~Sun, Piotr Koniusz, and Zhen Wang.
\newblock Fisher-bures adversary graph convolutional networks.
\newblock {\em Conference on Uncertainty in Artificial Intelligence},
  115:465--475, 2019.

\bibitem{DBLP:journals/corr/AdversaGA}
Susheel Suresh, Pan Li, Cong Hao, and Jennifer Neville.
\newblock Adversarial graph augmentation to improve graph contrastive learning.
\newblock {\em CoRR}, abs/2106.05819, 2021.

\bibitem{DBLP:conf/nips/whatGoodVeiw}
Yonglong Tian, Chen Sun, Ben Poole, Dilip Krishnan, Cordelia Schmid, and
  Phillip Isola.
\newblock What makes for good views for contrastive learning?
\newblock In {\em Advances in Neural Information Processing Systems 33: Annual
  Conference on Neural Information Processing Systems 2020, NeurIPS 2020},
  2020.

\bibitem{velickovic2019deep}
Petar Velickovic, William Fedus, William~L Hamilton, Pietro Li{\`o}, Yoshua
  Bengio, and R~Devon Hjelm.
\newblock Deep graph infomax.
\newblock {\em ICLR (Poster)}, 2019.

\bibitem{wang2019implicit}
Yulin Wang, Xuran Pan, Shiji Song, Hong Zhang, Gao Huang, and Cheng Wu.
\newblock Implicit semantic data augmentation for deep networks.
\newblock {\em Advances in Neural Information Processing Systems},
  32:12635--12644, 2019.

\bibitem{wei2019eda}
Jason Wei and Kai Zou.
\newblock Eda: Easy data augmentation techniques for boosting performance on
  text classification tasks.
\newblock {\em arXiv preprint arXiv:1901.11196}, 2019.

\bibitem{wu2021self}
Jiancan Wu, Xiang Wang, Fuli Feng, Xiangnan He, Liang Chen, Jianxun Lian, and
  Xing Xie.
\newblock Self-supervised graph learning for recommendation.
\newblock In {\em Proceedings of the 44th International ACM SIGIR Conference on
  Research and Development in Information Retrieval}, pages 726--735, 2021.

\bibitem{yang2020featurenorm}
Menglin Yang, Ziqiao Meng, and Irwin King.
\newblock Featurenorm: L2 feature normalization for dynamic graph embedding.
\newblock In {\em 2020 IEEE International Conference on Data Mining (ICDM)},
  pages 731--740. IEEE, 2020.

\bibitem{yang2022hrcf}
Menglin Yang, Min Zhou, Jiahong Liu, Defu Lian, and Irwin King.
\newblock Hrcf: Enhancing collaborative filtering via hyperbolic geometric
  regularization.
\newblock In {\em Proceedings of the ACM Web Conference 2022}, pages
  2462--2471, 2022.

\bibitem{you2020graph}
Yuning You, Tianlong Chen, Yongduo Sui, Ting Chen, Zhangyang Wang, and Yang
  Shen.
\newblock Graph contrastive learning with augmentations.
\newblock {\em Advances in Neural Information Processing Systems},
  33:5812--5823, 2020.

\bibitem{yu2022graph}
Junliang Yu, Hongzhi Yin, Xin Xia, Tong Chen, Lizhen Cui, and Nguyen Quoc~Viet
  Hung.
\newblock Are graph augmentations necessary? simple graph contrastive learning
  for recommendation.
\newblock {\em arXiv preprint arXiv:2112.08679}, 2022.

\bibitem{DBLP:conf/www/ZhangZMKK22}
Yifei Zhang, Hao Zhu, Ziqiao Meng, Piotr Koniusz, and Irwin King.
\newblock Graph-adaptive rectified linear unit for graph neural networks.
\newblock In {\em {WWW} '22: The {ACM} Web Conference 2022, Virtual Event,
  Lyon, France, April 25 - 29, 2022}, pages 1331--1339. {ACM}, 2022.

\bibitem{zhu2021refine}
Hao Zhu and Piotr Koniusz.
\newblock Refine: Random range finder for network embedding.
\newblock In {\em ACM Conference on Information and Knowledge Management},
  2021.

\bibitem{zhu2021simple}
Hao Zhu and Piotr Koniusz.
\newblock Simple spectral graph convolution.
\newblock In {\em International Conference on Learning Representations}, 2021.

\bibitem{zhu2021contrastive}
Hao Zhu, Ke~Sun, and Peter Koniusz.
\newblock Contrastive laplacian eigenmaps.
\newblock {\em Advances in Neural Information Processing Systems}, 34, 2021.

\bibitem{zhu2020deep}
Yanqiao Zhu, Yichen Xu, Feng Yu, Qiang Liu, Shu Wu, and Liang Wang.
\newblock Deep graph contrastive representation learning.
\newblock {\em arXiv preprint arXiv:2006.04131}, 2020.

\bibitem{zhu2021graph}
Yanqiao Zhu, Yichen Xu, Feng Yu, Qiang Liu, Shu Wu, and Liang Wang.
\newblock Graph contrastive learning with adaptive augmentation.
\newblock In {\em Proceedings of the Web Conference 2021}, pages 2069--2080,
  2021.

\end{thebibliography}
